\documentclass[conference]{IEEEtran}
\IEEEoverridecommandlockouts

\usepackage{cite}
\usepackage{amsmath,amssymb,amsfonts}
\usepackage{amsthm}
\usepackage{algorithm}
\usepackage{algpseudocode}
\usepackage{graphicx}
\usepackage{textcomp}
\usepackage{xcolor}
\usepackage{caption}
\usepackage{subcaption}
\usepackage{multicol, blindtext}

\newtheorem{theorem}{Theorem}
\newtheorem{lemma}{Lemma}
\def\BibTeX{{\rm B\kern-.05em{\sc i\kern-.025em b}\kern-.08em
    T\kern-.1667em\lower.7ex\hbox{E}\kern-.125emX}}
\begin{document}

\title{FedDIP: Federated Learning with Extreme \\Dynamic Pruning and Incremental Regularization}
\author{
\IEEEauthorblockN{Qianyu~Long\IEEEauthorrefmark{1},
Christos~Anagnostopoulos\IEEEauthorrefmark{1},
Shameem~Puthiya~Parambath\IEEEauthorrefmark{1}, and
Daning~Bi\IEEEauthorrefmark{2}}

\IEEEauthorblockA{\IEEEauthorrefmark{1}
\textit{School of Computing Science}, \textit{University of Glasgow, UK} \\
2614994L@student.gla.ac.uk, \{christos.anagnostopoulos, sham.puthiya\}@glasgow.ac.uk}

\IEEEauthorblockA{\IEEEauthorrefmark{2}
\textit{College of Finance and Statistics}, \textit{Hunan University, China} \\
daningbi@hnu.edu.cn}
}

\maketitle
\begin{abstract}
Federated Learning (FL) has been successfully adopted for distributed training and inference of large-scale Deep Neural Networks (DNNs). 
However, DNNs are characterized by an extremely large number of parameters, 
thus, yielding significant challenges in exchanging these parameters among distributed nodes and managing the memory. Although recent DNN compression methods (e.g., sparsification, pruning) tackle such challenges, they do not holistically consider an adaptively controlled reduction of parameter exchange while maintaining high accuracy levels. We, therefore, contribute with a novel FL framework (coined FedDIP), which combines (i) dynamic model pruning with error feedback to eliminate redundant information exchange, which contributes to significant performance improvement, with (ii) incremental regularization that can achieve \textit{extreme} sparsity of models. We provide convergence analysis of FedDIP and report on a comprehensive performance and comparative assessment against state-of-the-art methods using benchmark data sets and DNN models. Our results showcase that FedDIP not only controls the model sparsity but efficiently achieves similar or better performance compared to other model pruning methods adopting incremental regularization during distributed model training. The code is available at : https://github.com/EricLoong/feddip.
\end{abstract}

\newtheorem{remark}{Remark}

\section{Introduction}
Federated Learning (FL) \cite{mcmahan2017communication} is a prevalent \textit{distributed learning} paradigm due to its ability to tackle learning at scale. FL plays a significant role in large-scale predictive analytics by enabling the decentralization of knowledge discovery. FL contributes towards privacy preservation, which overcomes fundamental issues of data governance and ownership \cite{Kaissis2020}. 
Distributed training and deploying large-scale Machine Learning (ML) models, i.e., Deep Neural Networks (DNNs), impose significant challenges due to the huge volumes of training data, \textit{large} models, and diversity in data distributions.

Distributed computing nodes, mainly located at the network edge being as close to data sources as possible, collaboratively engineer ML models rather than depending on collecting all the data to a centralized location (data center or Cloud) for training \cite{yang2019federated}. This computing paradigm coined Edge Computing, has been successfully applied to various predictive modeling, mining, and analytics applications, e.g., in finance \cite{long2020federated}, healthcare \cite{xu2021federated} and wireless sensor networks \cite{niknam2020federated}. 

DNNs are characterized by an extremely large number of parameters. 
For instance, the Convolutional Neural Networks (CNN) ResNet50 \cite{he2016deep} and VGG16 \cite{simonyan2014very} consist of 27 and 140 million parameters, respectively, while generative AI models, like GPT-2 have more than 1.5 billion parameters \cite{radford2019language}. Evidently, this places a great burden on distributed computing nodes when exchanging model parameters during training, tuning, and inference. 

Model size reduction (pruning) methods, e.g., \cite{Aji2017}, \cite{alistarh2018convergence}, \cite{Jiang2022} aim to retain the prediction accuracy while reducing the communication overhead by decreasing the number of model parameters exchanged among nodes. However, most pruning methods focus on the compression of model gradients. Even though they yield high compression rates, they do not achieve significantly compact models for exchange. 
But in general, methods that can produce compact models along with significant redundancy in the number of DNN weights by sophisticatedly pruning the weights are deemed appropriate \cite{8736011}. 
In contrast to model gradient compression, model weight compression significantly shrinks the model size by setting most of the weights to zero. This is desirable for eliminating redundancy in model exchange during distributed knowledge extraction. But often such models result in performance degradation. 
Therefore, the question we are addressing is: \textit{How to effectively introduce model pruning mechanisms in a decentralized learning setting that is capable of achieving extremely high compression rates while preserving optimal predictive performance?} We contribute with an efficient method based on \textit{dynamic pruning with error feedback} and \textit{incremental regularization}, coined FedDIP. FedDIP's novelty lies in the principle of adapting dynamic pruning in a decentralized way by pushing unimportant weights to zeros (extreme pruning) whilst maintaining high accuracy through incremental regularization. To the best of our knowledge, FedDIP is the first approach that combines incremental regularization and extreme dynamic pruning in FL. 

The paper is organized as follows: Section \ref{sec:RW} reports on related work and our contribution. Section \ref{sec:prelimiaris} provides preliminaries in FL and model pruning methods. Section \ref{sec:Framework} elaborates on the FedDIP framework, while Section \ref{sec:CA} reports on the theoretical properties of FedDIP and convergence analysis. Our experimental results in Section \ref{sec:results} showcase the efficiency of FedDIP in distributed learning. Section \ref{sec:conclusion} concludes the paper with future research directions.

\section{Related Work \& Contribution}
\label{sec:RW}
\subsection{Model Gradient \& Model Weight Sparsification}
Expensive and redundant sharing of model weights is a significant obstacle in distributed learning \cite{li2020federated}.
The size of the exchanged models among nodes can be reduced by compression and sparsification. The work in \cite{alistarh2018convergence} adopts magnitude selection on model gradients to yield sparsification when using Stochastic Gradient Descent (SGD). Instead of dense updates of weights, \cite{Aji2017} proposed a distributed SGD that keeps 1\% of the gradients 
by comparing their magnitude values. The method in \cite{strom2015scalable} scales up SGD training of DNN via controlling the rate of weight update per individual weight. \cite{konevcny2018randomized} develops encoding SGD-based vectors achieving reduced communication overhead. \cite{jiang2018linear} proposed the periodic quantized averaging SGD strategy that attains similar model predictive performance while the size of shared model gradients is reduced $95\%$. In \cite{Lin2018}, the authors argued that 99\% of gradients are redundant and introduced a deep gradient compression method, which achieves compression rates in the range 270-600 with sacrificing accuracy. The \textit{gTop-k} gradient sparsification method in \cite{Shi2019} reduces communication cost based on the \textit{Top-k} method in \cite{Lin2018}. \cite{sun2019communication} develops a method based on \cite{chen2018lag} that adaptively compresses the size of exchanged model gradients via quantization.

In contrast to gradient sparsification, the shrinkage of the \textit{entire} model size is of paramount importance in distributed learning. It not only eliminates communication redundancy during training but also enables less storage and inference time, which makes FL welcome in distributed knowledge systems. However, so far, only centralized learning adopts model compression via, e.g., weight pruning, quantization, low-rank factorization, transferred convolutional filters, and knowledge distillation \cite{Cheng2017}, with pruning being our focus in this work. SNIP \cite{lee2018snip} introduces a method that prunes a DNN model once (i.e., prior to training) based on the identification of important connections in the model. \cite{He2017} proposes a centralized two-step method that prunes each layer of a DNN via regression-based channel selection and least squares reconstruction. The method in \cite{Zhang2018} prunes CNNs centrally using the Alternating Direction Method of Multipliers (ADMM). Following \cite{Zhang2018}, the PruneTrain method \cite{Lym2019} uses structured group-LASSO regularization to accelerate CNN training in a centralized location only. The DPF \cite{Lin2020} method allows dynamic management of the model sparsity with a feedback mechanism that re-activates pruned weights. 

\subsection{Contribution}
Most of the approaches in FL take into account only the communication overhead and thus adopt gradient sparsification. Nonetheless, weight sparsification is also equally important and can lead to \textit{accurate} distributed sparse models.
Such sparse models are lightweight and, thus, suitable for storage, transfer, training, and fast inference. As shown in \cite{yao2019federated}, model weights and gradients averaging policies are \textit{equivalent} only when the local number of model training epochs equals one. FedDIP tries to bridge the gap of weights average pruning in FL by obtaining highly accurate sparse models through incremental regularization and reducing communication during training through dynamic pruning. 

To the best of our knowledge in distributed learning, PruneFL \cite{Jiang2022} FedDST \cite{bibikar2022federated} and LotteryFL \cite{li2021lotteryfl} methods attempt model pruning. However, LotteryFL focuses on a completely different problem from ours. LotteryFL tries to \textit{discover sparse local} sub-networks (a.k.a. Lottery Ticket Networks) of a base DNN model. In contrast, FedDIP searches for a \textit{sparse global} DNN model with mask readjustments on a central server, as we will elaborate on later. PruneFL starts with a pre-selected node to train a global shared mask function, while FedDIP generates the mask function with weights following the Erd\H{o}s-Ren\'eyi-Kernel (ERK) distribution \cite{evci2020rigging}, as we will discuss in the later sections. FedDST, as proposed by Bibikar et al., initially derives a pruning mask based on the ERK distribution. Subsequent stages involve layerwise pruning on the global model. The method ensures efficient training through a prune-regrow procedure, which maintains a local sparse mask, particularly under non-iid data distributions. 
Our technical contributions are:
\begin{itemize}
    \item  An innovative federated learning paradigm, coined FedDIP, combines extreme sparsity-driven model pruning with incremental regularization. 
    \item FedDIP achieves negligible overhead keeping accuracy at the same or even higher levels over extremely pruned models. 
    \item Theoretical convergence and analysis of FedDIP.
    \item A comprehensive performance evaluation and comparative assessment of FedDIP with benchmark i.i.d. and non-i.i.d. datasets and various DNN models. Our experimental results reveal that FedDIP, in the context of high model compression rates, delivers superior prediction performance compared to the baseline methods and other approaches found in the literature, specifically, FedAvg \cite{mcmahan2017communication}, PruneFL \cite{Jiang2022}, PruneTrain \cite{Lym2019}, FedDST \cite{bibikar2022federated}, DPF \cite{Lin2020}, and SNIP \cite{lee2018snip}.  
\end{itemize}
\begin{table}[htbp]
    \centering
    \begin{tabular}{ |l|l|  }
\hline
 \textbf{Notations} &\textbf{Definition}\\
 \hline
  \hline
  $N,K$ & $N$: total number of nodes, where  $K < N$ nodes \\
  & participated in each training round\\
     \hline
    \multicolumn{2}{|l|}{$n$ indexes a node; $z$ indexes a DNN layer; $n \in [N], z \in [Z]$} \\
    \multicolumn{2}{|l|}{$[N]$ abbreviates the integer sequence $1, 2, \ldots, N$}\\
    \hline
 $\mathcal{D}_{n}, D_{n}$ & Dataset and its size on node $n$. \\
 $(\mathbf{x},y) \in \mathcal{D}_{n}$ & $\mathbf{x},y$ are features and labels in node $n$'s dataset\\ 
 $f(\cdot), \nabla f(\cdot)$&  Loss function and its derivative \\
 $\rho_{n}, \eta$ & Weight percentage and learning rate\\
 $\boldsymbol{\omega}_{G}$, $\boldsymbol{\omega}_{n}$, $\boldsymbol{\omega}_{n}^{\prime}$ & Global, local and pruned local model parameters\\
 $T, E_{l}, \tau, \ell$ & Global and local rounds, global and local epochs\\
 $\lambda$ & Regularization hyperparameter\\ 
$\odot$     & Element-wise (Hadamard) product\\
$s_{0}$, $s_{t}$, $s_{p}$ & initial sparsity, sparsity at round $t$, final sparsity\\
\hline
\end{tabular}
\caption{Table of Notations}
\label{table:notations}
\end{table}

\begin{figure}[!ht]
    \centering
    {
     \includegraphics[trim=8cm 3cm 9cm 5cm, clip, width=0.37\textwidth]{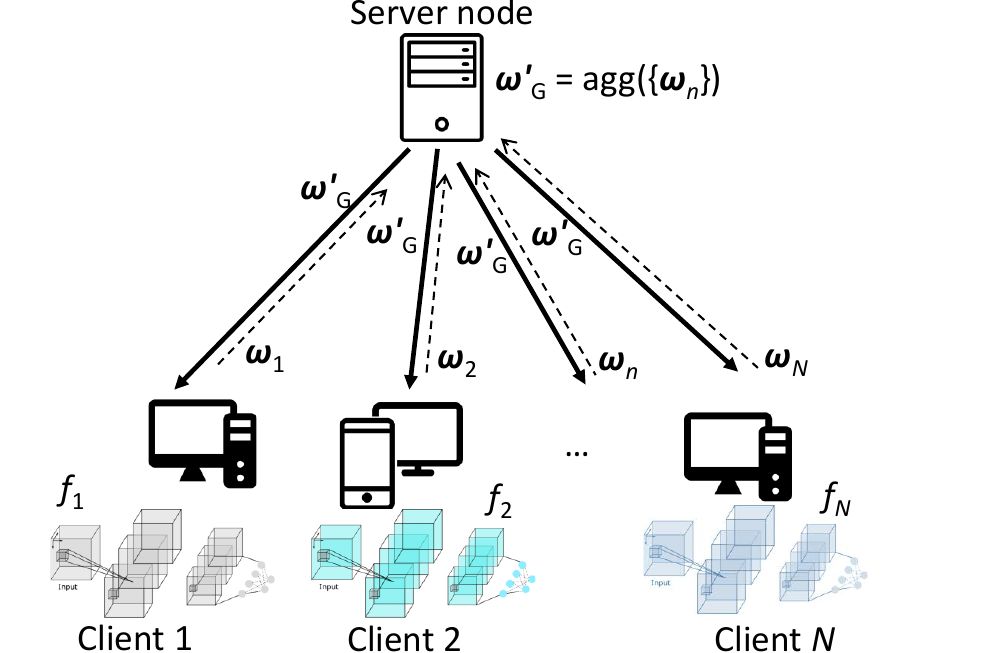}
    }
    \caption{
\textbf{Illustration of the FedDIP framework:} \\
(1) During the downlink phase, the pruned global model \( \mathcal{\omega}^{\prime}_{G} \) is broadcasted to participating clients. \\
(2) In the uplink phase, each selected client communicates its local dense model \( \mathcal{\omega}_{n} \) back to the server for aggregation. \\
(3) The global mask \( \mathbf{m}_{G} \) is derived from the global model, directing the sparse training (DPF) across clients.} 
    \label{fig:system}
\end{figure}

\section{Preliminaries}
\label{sec:prelimiaris}
\subsection{Federated Learning}
For the general notations and definitions, please refer to Table \ref{table:notations}. Consider a distributed learning system involving a set of $N$ nodes (clients) $\mathcal{N} = \{1, 2, \ldots, N\}$.
Let $\mathcal{D}_{n} = \{(\mathbf{x},y)\}$ be the local dataset associated with a node $n \in \mathcal{N}$ such that $\mathbf{x} \in \mathcal{X} \subset \mathbb{R}^{d}$, $y \in \mathcal{Y} \subset \mathbb{R}$, and $D_{n} = |\mathcal{D}_{n}|$.
In the standard FL setting, given a subset of $K < N$ nodes $\mathcal{N}_{c} \subset \mathcal{N}$, the local loss is given by:
\begin{equation}
\label{eq:fedavg_local}
    f_{n}(\boldsymbol{\omega}) = 
    \frac{1}{D_{n}} 
    \sum_{(\mathbf{x},y) \in \mathcal{D}_{n}}\mathcal{L}( \mathcal{G}(\boldsymbol{\omega},\mathbf{x}),y)
\end{equation}
where $\boldsymbol{\omega}$ is the model parameter, $\mathcal{G}$ is the discriminant function that maps the input space to output space and  $\mathcal{L}$ is a loss function that measures the quality of the prediction, e.g., mean-squared-error, maximum likelihood, cross-entropy loss.
The global loss function for all the selected nodes $n \in \mathcal{N}_{c}$ is:
\begin{equation}
    f(\boldsymbol{\omega}) = \sum_{n \in \mathcal{N}_{c}}\rho_{n} f_{n}(\boldsymbol{\omega}), \text{ where }
    \rho_{n}= \frac{D_{n}}{\sum_{j \in \mathcal{N}_{c}}D_{j}}.
\end{equation}
The model training process spans periodically over $T$ global rounds 
with $L$ local rounds. Let $t \in = \{0, 1, \ldots, T-1\}$ be a discrete-time instance during the training process. Then, $\tau=\lfloor\frac{t}{L}\rfloor L$ is the start time of the current global epoch. At $\tau$, the nodes (clients) receive updated aggregated weights $\bar{\boldsymbol{\omega}}^{\tau}$ from the node responsible for aggregating the nodes' model parameters, a.k.a. the server node. 
The local training at client $n$ at local epoch $l = 1, \ldots, L$ proceeds as:
\begin{equation}
\label{eq:local_training}
    \boldsymbol{\omega}^{(\tau+l)+1}_{n} = \boldsymbol{\omega}^{\tau+l}_{n}-\eta_{\tau+l}
    \nabla f_{n}(\boldsymbol{\omega}^{\tau+l}_{n}),
\end{equation}
where $\eta \in(0,1)$ is the learning rate. The weight averaging policy on the server node can be written as:
\begin{eqnarray}
\label{eq:cloud_avg}
    \bar{\boldsymbol{\omega}}^{\tau} & = & \sum_{n \in \mathcal{N}}\rho_{n}\boldsymbol{\omega}_{n}^{\tau}.
\end{eqnarray}

\subsection{Model Pruning}
\label{sect:dis_adap_prune}
In centralized learning systems (e.g., in Cloud), where all data are centrally stored and available, the model pruning \cite{Zhu2017} aims to sparsify various connection matrices that represent the weights of the DNN models. 
Notably, \textit{sparsity}, hereinafter noted by $s \in [0,1]$, indicates the proportion of non-zero weights among overall weights.
A 100\% sparse ($s=1$) model indicates that all the weights are negligible (their values are close to $0$), 
while a 0\% sparse ($s=0$) model stands for the full model with original weight values. 
Typically, the reduction of the number of nonzero weights (pruning) of a DNN model is achieved using \textit{mask functions}. A mask function $\mathbf{m}$ acts like an indicator function that decides whether the parameter/weight at a certain position in a layer of a DNN model is zero or not. The model pruning based on mask functions requires a criterion to select the parameters to prune. The most common pruning criterion considers the absolute value of the weights of each parameter in a layer. Generally, a parameter is removed from the training process if its absolute value of the weight is less than a predefined threshold. 

On the other hand, model pruning in FL is vital in light of reducing communication cost in \textit{each} training round. Moreover, the global number of rounds should be reduced as this significantly contributes to the overall communication overhead. Hence, in FL, pruning aims at \textit{extreme} model compression rates, i.e., $s \geq 0.8$ with a relatively small compromise in prediction accuracy. 
It is then deemed appropriate to introduce a distributed and adaptive pruning method with relatively high and controlled DNN model sparsity, which reduces communication costs per round along with ensuring convergence under high sparsity with only a marginal decrease in prediction accuracy.

The pruning techniques are typically categorized into three: 
\textit{pruning before training} (e.g., SNIP \cite{lee2018snip}), \textit{pruning during training} (e.g., PruneTrain \cite{Lym2019}, FedDST \cite{bibikar2022federated}, DPF \cite{Lin2020} and PruneFL \cite{Jiang2022}),  and \textit{pruning after training}. In this work, we concentrate on the two former techniques, which deal with efficient model training. The \textit{pruning after training} approach offers limited utility in the context of distributed learning. The two commonly employed techniques for pruning are: (i) Regularization-based Pruning (RP) and (ii) Importance-based Pruning (IP) \cite{Wang2021}. The interested reader may refer to \cite{He2017,Zhang2018,Wang2021} and the references therein for a comprehensive survey of RP and IP techniques. 
RP uses intrinsic sparsity-inducing properties of $L_1$ (Manhattan distance) and $L_2$ (Euclidean distance) norms to limit the  \textit{importance} of different model parameters. The sparsity-inducing norms constrain the weights of the unimportant parameters to small absolute values during training. Moreover, RP can effectively constrain the weights into a sparse model space via tuning the regularization hyperparameter $\lambda$.
Whereas in IP, parameters are pruned purely based on predefined formulae that are defined in terms of the weights of the parameters or the sum of the weights. IP techniques were originally proposed in the unstructured pruning settings that can result in sparse models not capable of speeding up the computation. Even though RP techniques are considered superior to IP techniques, they struggle with two fundamental challenges: (\textbf{C1}) The first challenge pertains to controlling the sparsity value $s$ during pruning. For example, in PruneTrain \cite{Lym2019}, employing a pruning threshold value of $10^{-4}$ to eliminate model parameters does not guarantee the delivery of a sparse model. (\textbf{C2}) The second challenge is dynamically tuning a regularization parameter $\lambda$. A large $\lambda$ leads to model divergence during training, as the model may excessively lean towards penalty patterns. By adding regularization terms in DNN training traditionally aims for overfitting issues. 
However, additional regularization for prunable layers is required for RP, which is the core difference between traditional training and RP-based training.
\section{The FedDIP Framework}
\label{sec:Framework}
The proposed FedDIP framework integrates extreme dynamic pruning \textit{with} error feedback and incremental regularization in distributed learning environments. Figure \ref{fig:system} illustrates a schematic representation of the FedDIP, which will be elaborated on in this section. FedDIP attempts to effectively train pruned DNN models across collaborative clients ensuring convergence by addressing the two challenges \textbf{C1} and \textbf{C2} prevalent in RP-based methods discussed in Section \ref{sect:dis_adap_prune}. 
The dynamic pruning method (DPF) in \cite{Lin2020} demonstrates improved performance in comparison with other baselines under high sparsity. Given the SGD update scheme, the model gradient in DPF is computed on the pruned model as:
\begin{equation}
\label{eq:DPF}
\boldsymbol{\omega}_{t+1} = \boldsymbol{\omega}_{t}-\eta_{t}\nabla f(\boldsymbol{\omega}_{t}^{\prime}) = \boldsymbol{\omega}_{t}-\eta_{t}\nabla f(\boldsymbol{\omega}_{t}\odot \mathbf{m}_{t}),
\end{equation}
taking into account the error feedback (analytically):
\begin{equation}\label
{eq:DPF_explanation}
    \boldsymbol{\omega}_{t+1} = \boldsymbol{\omega}_{t}-\eta_{t}\nabla f(\boldsymbol{\omega}_{t}+\mathbf{e}_{t}),
\end{equation} 
where $\mathbf{e}_{t} = \boldsymbol{\omega}_{t}^{\prime}-\boldsymbol{\omega}_{t}$. In (\ref{eq:DPF}), $\odot$ represents the Hadamard (element-wise) product between the two model weights, $\boldsymbol{\omega}_{t}$ represents the entire model parameters, $\boldsymbol{\omega}_{t}^{\prime}$ represents the pruned model parameters, and $\mathbf{m}$ is the adopted mask function used for pruning as in, e.g., in \cite{Jiang2022}, \cite{Lym2019}, and \cite{Lin2020}. The mask is applied on the model parameters $\boldsymbol{\omega}_{t}$ to eliminate weights according to the magnitude of each weight, thus, producing the pruned $\boldsymbol{\omega}_{t}^{\prime}$. Applying the gradient, in this case, allows recovering from errors due to premature masking out of important weights, i.e., the rule in (\ref{eq:DPF}) takes a step that best suits the pruned model (our target). 
In contrast, all the pruning methods adopted in FL, e.g., \cite{Jiang2022}, led to sub-optimal decisions by adopting the rule:
\begin{equation}
\label{eq:prune_both}
\boldsymbol{\omega}_{t+1} = \boldsymbol{\omega}_{t}^{\prime}-\eta_{t}\nabla f(\boldsymbol{\omega}_{t}^{\prime}).
\end{equation}
One can observe that the update rule in (\ref{eq:DPF}) retains more information, as it only computes gradients of the pruned model, compared to the update rule in (\ref{eq:prune_both}). This is expected to yield superior performance under high sparsity. 

Moreover, it is known that the multi-collinearity\footnote{In multi-collinearity, two or more independent variables are highly correlated in a regression model, which violates the \textit{independence} assumption.} challenge is alleviated by the Least Absolute Shrinkage and Selection Operator (LASSO). LASSO performs simultaneous variable selection and regularisation \cite{Fu2000}. LASSO adds the $L_1$ regularization term to the regression loss function, providing a solution to cases where the number of model parameters is significantly larger than the available observations. Apparently, this is the case in DNNs, which typically involve millions of parameters with only tens of thousands of observations. 
The two challenges reported in Section~\ref{sect:dis_adap_prune} deal with selecting appropriate dynamic policies for sparsity control and regularization hyperparameter $\lambda$. To address the challenge \textbf{C1}, we 
dynamically drop the least $s \cdot 100\%$ percentile according to weights magnitude. The challenge \textbf{C2} is addressed by incrementally increasing the regularization parameter departing from the principles of LASSO regression. 
It is also evidenced in \cite{Wang2021} that growing regularization benefits pruning.  
Based on these observations, we establish the FedDIP algorithm to maintain the predictive model performance under extreme sparsity with incremental regularization and dynamic pruning. To clarify terminology, we refer to our algorithm that directly applies dynamic pruning as `FedDP' (addressing challenge \textbf{C1}), while `FedDIP' represents the variant that also adds incremental regularization (addressing both challenges \textbf{C1} and \textbf{C2}). Collectively, we refer to these variants as `FedD(I)P'. 
Each node $n \in \mathcal{N}$ first trains a local sparse DNN model, which contains weights with relatively small magnitudes (see also Fig. \ref{fig:system}). Then, the node $n$ optimizes the proposed \textit{local incrementally regularized loss function} at round $t$ as: 
\begin{equation}
\label{eq:fed_dip}
        f_{n}(\boldsymbol{\omega}_{t}) = \frac{1}{D_{n}} \sum_{(\mathbf{x},y) \in \mathcal{D}}\mathcal{L}(G(\boldsymbol{\omega}_{t},\mathbf{x}),y) + \lambda_{t}\sum_{z=1}^{Z}\|\boldsymbol{\omega}_{t}^{(z)}\|_{2},
\end{equation}
where the step $t$ dependent regularization parameter $\lambda_{t}$ controls the degree of model shrinkage, i.e., the \textit{sparsity}, and $Z$ is the number of the DNN layers (this, of course, depends on the DNN architecture; in our experiments, it is the sum of convolutional and fully connected layers).
The norm $\|\boldsymbol{\omega}^{(z)}\|_{2} = (\sum_{k} \lvert \omega^{(z)}_{k} \rvert^{2})^{1/2}$ is the $L_{2}$ norm of the pruned $z^{th}$ layer of model weights $\boldsymbol{\omega}^{(z)}$. We then introduce the incremental regularization over $\lambda_{t}$ based on the schedule: 
\begin{equation}
\label{eq:increase_lambda}
  \lambda_{t} = 
  \begin{cases} 
   0 & \text{if } 0 \leq t < \frac{T}{Q} \\
   \vdots & \vdots \\
   \frac{\lambda_{\text{max}} \cdot (i-1)}{Q} & \text{if } \frac{(i-1)T}{Q} \leq t < \frac{iT}{Q} \\
   \vdots & \vdots \\
   \frac{\lambda_{\text{max}}(Q-1)}{Q} & \text{if } \frac{(Q-1)T}{Q} \leq t \leq T \\
  \end{cases}
\end{equation}
with quantization step size $Q>0$. The influence of $Q$ on regularization is controlled by adapting $\lambda_{max}$. Such step size divides the regularization parameter space from $\frac{\lambda_{\max}}{Q}$ to $\lambda_{\max}$ to achieve a gradual increase of regularization at every $\frac{T}{Q}$ rounds. In addition, each node $n$ 
adopts dynamic pruning to progressively update its local model weights $\boldsymbol{\omega}_{n}^{\tau+L}$ to optimize (\ref{eq:fed_dip}) as: 
\begin{equation}
\label{eq:local_training_dpf}
    \boldsymbol{\omega}^{\tau,l+1}_{n} = \boldsymbol{\omega}^{\tau,l}_{n}-\eta_{\tau}
    \nabla f_{n}(\boldsymbol{\omega}^{\prime(\tau,l)}_{n}),
\end{equation}
where $\boldsymbol{\omega}^{\prime(\tau+l)}_{n}$ is obtained through pruning based on a global mask function $\mathbf{m}_{\tau}$ generated by the server node. 
Moreover, our gradual pruning policy modifies the sparsity update policy per round from \cite{Li2020} by incrementally updating the sparsity as:
\begin{equation}
\label{eq:gradual_pruning}
s_{t} = s_{p}+(s_{0}-s_{p})\Big(1-\frac{t}{T}\Big)^{3},
\end{equation}
where $s_{t}$ represents the sparsity applied to the model pruning at round $t$, $s_{0}$ is the initial sparsity, and $s_{p}$ is the desired/target sparsity. 
Notably, in our approach 
$s_{0}$ is strictly non-zero; this can be a moderate sparsity of $s_{0}=0.5$. Such adaptation differentiates our method from \cite{Li2020}, where $s_{0}=0$. In essence, we permit the sparsity to increment from moderate to extreme levels throughout the process. If considering $s_{0} > 0$, the layer-wise sparsity of the initial mask follows the ERK distribution introduced in \cite{evci2020rigging}.
At the end of a local epoch $l$, the server node collects $K < N$ model weights $\boldsymbol{\omega}_{n}^{\tau+l}$ from the selected nodes $n \in \mathcal{N}_{c}$, and calculates the global weights average as:
\begin{equation}
\label{eq:prune_omega}
    \bar{\boldsymbol{\omega}}^{\tau+l}_{G} = \sum_{n \in \mathcal{N}}\rho_{n}\boldsymbol{\omega}_{n}^{\tau+l}.
\end{equation}
In addition, the $\mathbf{m}_{\tau}$ mask function is generated based on pruning on $\bar{\boldsymbol{\omega}}^{\tau+l}_{G}$ with current sparsity $s_{\tau}$. 
The FedDIP process is summarized in Algorithm \ref{algorithm:feddip}, where \textit{only} pruned models are exchanged from server to nodes, while pruning is \textit{locally} achieved in the clients. \textbf{Note:} FedDIP achieves data-free initialization and generalizes the DPF \cite{Lin2020} in dynamic pruning process. When we set initial $s_{0} = 0$ and no incremental regularization, i.e., $\lambda_{t} = 0$, $\forall t$, then FedDIP reduces to DPF. Moreover, we obtain our variant FedDP if we set $\lambda_{t} = 0$, $\forall t$ with $s_{0} > 0$ w.r.t. ERK distribution.

\begin{remark}

     \textbf{Trade-off between Pruning and Fine-tuning:} The FedDIP approach introduces a reconfiguration horizon, denoted as \( R \), during the model training phase to periodically update the mask function. Specifically, the mask function \( \mathbf{m}_{\tau} \) is updated at every \( R \) global round, i.e., when \( \tau \mod R = 0 \), to ensure a consistent and smooth accuracy learning curve. The value of this horizon is determined empirically. \textit{Potential Outcomes of Insufficient Pruning:} If the mask function remains unchanged throughout the horizon \( T \), there's a risk that the model could converge to a local optimum. \textit{Consequences of Insufficient Fine-tuning:} Conversely, if the mask function undergoes frequent updates, the changes in the model might not align with the alterations in the sparse model structure.

\end{remark}


\begin{remark}
\textbf{Integration of Incremental Regularization and DPF:} 
Differing from the approach in \cite{Wang2021}, which centralizes increasing penalty factors on pre-trained models, FedDIP initiates this from the outset within a distributed learning context. The integration of incremental regularization with DPF offers advantages, primarily because DPF obviates the need for post-pruning fine-tuning, making it preferable to one-shot pruning methods like SNIP. 
\end{remark}

\begin{algorithm}
\caption{The FedDIP Algorithm}
\label{algorithm:feddip}
\begin{algorithmic}[1]
 \renewcommand{\algorithmicrequire}{\textbf{Input: }}
 \renewcommand{\algorithmicensure}{\textbf{Output:}}
\Require $N$ nodes; $T$ global rounds; $E_{l}$ local rounds; initial and target sparsity $s_{0}$ and $s_{p}$; maximum regularization $\lambda_{\max}$; quantization step $Q$; reconfiguration horizon $R$
\Ensure Global pruned DNN model weights $\boldsymbol{\omega}_{G}^{\prime}$ 
\State \textbf{//Server initiliazation}
\If{$s_{0} > 0$}
\State Server initializes global mask $\mathbf{m}_{0}$ (ERK distribution)
\EndIf
\State \textbf{//Node update \& pruning}
\For{global round $\tau = 1, \ldots, T$}
        \State Server randomly selects $K$ nodes $\mathcal{N}_{c} \subset \mathcal{N}$
        \For{selected node $n \in \mathcal{N}_{c}$ \textbf{in parallel}}
        \State Receive pruned weights $\boldsymbol{\omega}_{G}^{\prime(\tau-1)}$ from server node
        \State Obtain mask $\mathbf{m}_{\tau-1}$ from $\boldsymbol{\omega}_{G}^{\prime(\tau-1)}$
        \State Train $\boldsymbol{\omega}_{n}^{\tau}$ over $E_{l}$ rounds on data $\mathcal{D}_{n}$ using (\ref{eq:local_training_dpf}) 
        \If{incremental regularization is chosen}
        \State Optimize (\ref{eq:fed_dip}) with incremental $\lambda_{\tau}$ in (\ref{eq:increase_lambda})
        \Else{}
        \State Optimize (\ref{eq:fedavg_local})
        \EndIf
        \EndFor
        \State \textbf{//Server update, aggregation \& reconfiguration}
        \State Server receives models and aggregates $\boldsymbol{\omega}_{G}^{\tau}$ in (\ref{eq:prune_omega})
        \If{$\tau \mod R == 0$}
        \State Reconfigure global mask $\mathbf{m}_{\tau}$ based on pruning $\boldsymbol{\omega}_{G}^{\tau}$
        \EndIf
        \State Server prunes global model with $\mathbf{m}_{\tau}$ and obtains $\boldsymbol{\omega}_{G}^{\prime(\tau)}$
        \State Server node returns $\boldsymbol{\omega}_{G}^{\prime(\tau)}$ to all nodes.
\EndFor
\end{algorithmic}
\end{algorithm}

\section{Theoretical \& Convergence Analysis}
\label{sec:CA}
\newtheoremstyle{remboldstyle}
  {}{}{\itshape}{}{\bfseries}{.}{.5em}{{\thmname{#1 }}{\thmnumber{#2}}{\thmnote{ (#3)}}}
\theoremstyle{remboldstyle}
\newtheorem{assumption}{Assumption}
\newtheorem{definition}{Definition}
In this section, we provide a theoretical analysis of FedDIP including the convergence Theorem \ref{theorem:1} ensuring stability in training models w.r.t. incremental regularization and dynamic extreme pruning. \textbf{\underline{Note for Proofs}:} \textit{The proofs of our Theorem \ref{theorem:1} and lemmas are in the Appendix \ref{sec:appendix}}

At each global round $t \in \{1, \ldots, T\}$, $K$ out of $N$ nodes participate, each one selected with probability $\rho_{n}$ aligned with \cite{Haddadpour2019},  \cite{Li2019} and $\sum_{n=1}^{N}\rho_{n}=1$. 
Let $\boldsymbol{\omega}_{n}^{t}$ and $\boldsymbol{\omega}_{n}^{\prime (t)}$ be the weights and pruned ones at round $t$ on node $n$, respectively, with 
\begin{equation}
\label{eq:ca:pr_1}
\boldsymbol{\omega}_{n}^{\prime (t)} = \boldsymbol{\omega}_{n}^{t} \odot \mbox{m}^{t}.
\end{equation}
Let also $\mathbf{v}_{n}^{t}$ and $\Tilde{\mathbf{v}}_{n}^{t}$ be the expected and estimated gradients at $t$, respectively, on node $n$. Based on $\boldsymbol{\omega}_{n}^{\prime (t)}$, we obtain: 
$\mathbf{v}_{n}^{\prime(t)}  = \nabla f(\boldsymbol{\omega}_{n}^{\prime (t)})$
while $\Tilde{\mathbf{v}}_{n}^{\prime(t)}$ is the estimated one. The global aggregated model for FedAvg is:
\begin{align}
\label{eq:aggrg_prob}
    \Bar{\boldsymbol{\omega}}^{t}&= \frac{1}{K}\sum_{n\in \mathcal{N}_{c}}\boldsymbol{\omega}^{t}_{n},
\end{align}
while before the server sends the model, it is pruned as
\begin{equation}
        \Bar{\boldsymbol{\omega}}^{\prime(t)} = \frac{1}{K}\sum_{n\in \mathcal{N}_{c}}\boldsymbol{\omega}^{t}_{n}\odot \mbox{m}^{t}.
\end{equation}
The global estimated aggregated gradient and expected global gradient, respectively, are:
\begin{equation}
\label{eq:grad_est}
    \mathbf{\Tilde{{v}}}^{t} = \frac{1}{K}\sum_{n\in \mathcal{N}_{c}} \Tilde{\mathbf{v}}_{n}^{t} \mbox{ and }   \mathbf{\Bar{v}}^{t} = \frac{1}{K}\sum_{n\in \mathcal{N}_{c}} \mathbf{v}_{n}^{t}.
\end{equation}
Similarly, for DPF, we have that:
\begin{equation}
\label{eq:pruned_gradients}
        \mathbf{\Tilde{{v}}}^{\prime(t)} = \frac{1}{K}\sum_{n\in \mathcal{N}_{c}} \Tilde{\mathbf{v}}_{n}^{\prime(t)} \text{ and }   
        \mathbf{\Bar{v}}^{\prime(t)} = \frac{1}{K}\sum_{n\in \mathcal{N}_{c}} \mathbf{v}_{n}^{\prime(t)}.
\end{equation}
In FedAvg, $\Bar{\boldsymbol{\omega}}^{t}$ is updated as:
$\Bar{\boldsymbol{\omega}}^{t+1} = \Bar{\boldsymbol{\omega}}^{t} - \eta_{t}\Tilde{\mathbf{v}}^{t}$,
while the update rule based on DPF at node $n$ is:
\begin{equation}
\label{eq:feddip_local_update}
 \boldsymbol{\omega}^{t+1}_{n} = \boldsymbol{\omega}^{t}_{n} - \eta_{t} \Tilde{\mathbf{v}}_{n}^{\prime(t)},
\end{equation}
where $\boldsymbol{\omega}^{t}_{n} = \boldsymbol{\bar{\omega}}^{\prime(t)}$. Similarly, $\Bar{\boldsymbol{\omega}}^{t+1}$ is updated as:
\begin{equation}\label{eq:feddip_glob}
       \Bar{\boldsymbol{\omega}}^{t+1} = \Bar{\boldsymbol{\omega}}^{\prime(t)} - \eta_{t}\Tilde{\mathbf{v}}^{\prime(t)}.
\end{equation}
\begin{definition}
\label{def:quality_prune}
     According to \cite{Lin2020}, the \textit{quality of pruning} is defined by the parameter $\delta_{t}\in [0,1]$ as:
\begin{equation}\label{eq:quality_prune}
    \delta_{t} := \frac{\|\boldsymbol{\omega}^{t}-\boldsymbol{\omega}^{\prime(t)}\|^{2}_{F}}{\|\boldsymbol{\omega}^{t}\|^{2}_{F}}
\end{equation}
\end{definition}
where $\|.\|^{2}_{F}$ is the square of Frobenius matrix norm. $\delta_{t}$ indicates the degree of information loss by pruning in terms of magnitude. A smaller $\delta_{t}$ stands for less information loss. 
\begin{definition}
\label{def:non_iid}
    Following the Definition $1$ in \cite{ca2021}, a measurement $\gamma$ of non-i.i.d. (non-independent and identically distributed) data is defined as follows:
\begin{equation}\label{eq:lambda_gradient}
    \gamma = \frac{\sum_{n=1}^{N}p_{n}\| \nabla f_{n}(\boldsymbol{\omega})\|^{2}}{\|\sum_{n=1}^{N}p_{n}\nabla f_{n}(\boldsymbol{\omega})\|^{2}},
\end{equation}
with $\gamma\geq 1$; $\gamma=1$ holds in i.i.d case.
\end{definition}

We list our assumptions for proving the convergence of FedDIP in the learning phase. 
\begin{assumption}
\label{assumption:alpha_convex}
\textbf{$L-$Smoothness}.
$,\forall \boldsymbol{\omega}^{t_{1}}, \boldsymbol{\omega}^{t_{2}}\in \mathbb{R}^{d}$, $L \in \mathbb{R}$
\begin{eqnarray*}
    f(\boldsymbol{\omega}^{t_{1}})\leq f(\boldsymbol{\omega}^{t_{2}})+(\boldsymbol{\omega}^{t_{1}}-\boldsymbol{\omega}^{t_{2}})^{\top}\nabla f(\boldsymbol{\omega}^{t_{2}})+\frac{L}{2}\|\boldsymbol{\omega}^{t_{1}}-\boldsymbol{\omega}^{t_{2}}\|^{2}
\end{eqnarray*}
\end{assumption}


\begin{assumption}
\label{assumption:mu_Lipschitz}
\textbf{$\mu-$Lipschitzness}.
$\forall \boldsymbol{\omega}^{t_{1}}, \boldsymbol{\omega}^{t_{2}}\in \mathbb{R}^{d}$ and $\mu \in \mathbb{R}$
\begin{equation}
    \|f(\boldsymbol{\omega}^{t_{1}})-f(\boldsymbol{\omega}^{t_{2}})\|\leq \mu\|\boldsymbol{\omega}^{t_{1}}-\boldsymbol{\omega}^{t_{2}}\|
\end{equation}
\end{assumption}
\begin{assumption}
\label{assumption:bounded_gradient}
\textbf{Bounded variance for gradients}. Following Assumption 3 in \cite{Li2019}, the local model gradients on each node $n$ are self-bounded in variance:
\begin{equation}
    \mathbb{E}[\|\Tilde{\mathbf{v}}_{n}^{t}-\mathbf{v}_{n}^{t}\|^{2}]\leq \sigma^{2}_{n}.
\end{equation}
\end{assumption}
\begin{assumption}\label{assumption_bounded_weighted_gradient}
\textbf{Bounded weighted aggregation of gradients.}
Following Assumption 4 in \cite{ca2021}, the aggregation of local gradients at time $t$ are bounded as:
\begin{equation}
    \|\sum_{n=1}^{N}\rho_{n}\mathbf{v}_{n}^{t}\|^{2}\leq G^{2},
\end{equation}
where $\sum_{n=1}^{N}\rho_{n}=1$ and $\sum_{n=1}^{N}\rho_{n}\mathbf{v}_{n}^{t}$ stands for the weighted aggregation of local gradients; $G \in \mathbb{R}$.
\end{assumption}

\begin{theorem}[FidDIP Convergence]
\label{theorem:1}
Consider the Assumptions \ref{assumption:alpha_convex}, \ref{assumption:mu_Lipschitz} and \ref{assumption:bounded_gradient}, Lemmas \ref{lemma1}, 
\ref{lemma2}, \ref{lemma3}, and let $\eta_{t}=\frac{1}{tL}$, $L>0$. Then, the convergence rate of the FedDIP process is bounded by:
\begin{multline}
\label{eq:theorem:1}
    \frac{1}{T}\sum_{t=1}^{T}\|\nabla f(\boldsymbol{\bar{\omega}}^{\prime(t)})\|^{2}\leq 2L\mathbb{E}(f(\boldsymbol{\omega}_{1})-f^{*}) +\\2L\sum_{t=1}^{T}[\mu\mathbb{E}[\sqrt{\delta_{t+1}}\|\boldsymbol{\bar{\omega}}^{t+1}\|]+\frac{\pi^{2}}{3L^{2}}\chi,
\end{multline}
where $f(\boldsymbol{\omega}_{1})$ and $f^{*}$ stand for the initial loss and the final convergent stable loss, with $\chi=\frac{(\gamma-1)L^{2}+L}{2K}\sum_{n=1}^{N}\rho_{n}\sigma^{2}_{n}+
\frac{(\gamma-1)\gamma E_{l}^{2}L^{2}G^{2}}{2}$, and $\gamma$ defined in Definition \ref{def:non_iid}.
\end{theorem}
\begin{proof}
Refer to `Note for Proofs' at the beginning of this section.
\end{proof}
In Theorem \ref{theorem:1}, the first term of the right-hand side of the inequality (\ref{eq:theorem:1}) denotes the gap between the initial and final loss, while $\chi$ goes to zero as $K \gg 1$ and the i.i.d. case assumption holds. This also suggests that non-i.i.d. case results in large boundaries. The quantity $\frac{1}{T}\sum_{t=1}^{T}\|\nabla f(\boldsymbol{\bar{\omega}}^{\prime(t)})\|^{2}$ is bounded by the loss produced by pruning. Overall, the convergence result shows that the $L_{2}$ norms of the pruned gradients parameters vanish over time, which indicates that a stable model is obtained at the end (recall, a stable gradient vector enables a small change on the model under SGD). 

\section{Experimental Evaluation}
\label{sec:results}
\subsection{Experimental Setup}
\textbf{Datasets and Models:} We experiment with the datasets \textit{Fashion-MNIST} \cite{Xiao2017}, \textit{CIFAR10}, and $\textit{CIFAR100}$ \cite{krizhevsky2009learning}. \textit{Fashion-MNIST} consists of $60,000$ training and $10,000$ test 28x28 grayscale images labeled from 10 classes. Both of \textit{CIFAR} datasets consist of $50,000$ training and $10,000$ test 32x32 color images; in \textit{CIFAR10} and \textit{CIFAR100} there are 10 classes (6000 images per class) and 100 classes (600 images per class), respectively. We consider the i.i.d. (independent and identically distributed) case to compare all the algorithms and extend FedDIP to be applied for non-i.i.d. cases. To test and compare the efficiency of FedDIP, we use different well-known CNN architectures: \textit{LeNet-5}\cite{LeCun2015}, \textit{AlexNet} \cite{Krizhevsky2014} and \textit{Resnet-18} \cite{he2016deep} as backbone (dense or unpruned) models, with the baseline FedAvg \cite{mcmahan2017communication} and the pruning baselines PruneFL \cite{Jiang2022}, PruneTrain \cite{Lym2019}, FedDST \cite{bibikar2022federated}, DPF \cite{Lin2020} (equivalent to FedDP as discussed above), and SNIP \cite{lee2018snip}. For the non-i.i.d. case, we adopt the pathological data partition method in \cite{mcmahan2017communication}, which assigns only two classes for each node. We merge FedDIP with FedProx \cite{MLSYS2020_38af8613}, a generalization and re-parametrization of FedAvg to address the heterogeneity of data (coined FedDIP+Prox), and compare with baseline FedAvg and FedProx. 
Our target is to evaluate FedDIP's accuracy, storage, and communication efficiency in FL environments under extreme sparsity. 

\textbf{Configurations:} Table \ref{table:configuration} details our configurations. For PruneFL, FedDST, and PruneTrain, we experimentally determined the optimal reconfiguration intervals $R$ to be $20$, $20$, and $1$, respectively, to ensure the \textit{best} possible model performance; the same for step size $Q$ for all models. Especially, the annealing factor for FedDST is set as $0.5$.
As SNIP prunes the model before training, the global mask is pruned via one-shot achieving the target sparsity $s_{p}$. We used grid-search to fix the penalty factor for PruneTrain ranging from $10^{-1}$ to $10^{-5}$ for different experiments. When necessary, other hyperparameters were set to match ours. In non-i.i.d. case, the penalty for the proximal term in FedProx is determined via grid-search ranging from $10^{-1}$ to $10^{-5}$. FedDIP+Prox adopts the optimal combination of penalty values for FedDIP and FedProx.
\begin{table*}[h]
\centering
\begin{tabular}{|l||l|l|l|}
\hline
Datasets      & Fashion-MNIST & CIFAR10 & CIFAR100 \\ \hline
DNN/CNN Model         & LeNet-5       & AlexNet & ResNet-18 \\ \hline
Number of pruning layers ($Z$)  & 5       & 8 &  18 \\ \hline
Initial learning rate ($\eta_{0}$)         &       $0.01$        &    $0.1$     &  $0.1$         \\ \hline
Number of clients per round ($K$)          & $5$ (out of $50$)            &  $5$ (out of $50$)     &  $5$ (out of $50$)        \\ \hline
Batchsize in SGD              &    $64$           &     $128$    &     $128$      \\ \hline
Initial sparsity ($s_{0}$)              &      $0.5$         &    $0.5$     &     $0.05$      \\ \hline
Global rounds ($T$)              &   $1,000$            &  $1,000$       &    $1,000$       \\ \hline
Reconfiguration interval ($R$)              &    $5$           &     $5$    &      $5$    \\ \hline
Regularization step size ($Q$)              &     $10$          & $10$        &  $10$         \\ \hline
Local round ($E_{l}$)              &     $5$          &    $5$     &      $5$     \\ \hline
Maximum penalty ($\lambda_{\max}$)              &  $10^{-3}$              &  $10^{-3}$       &    $5\cdot 10^{-3}$      \\ \hline
\end{tabular}
\caption{Configuration Table}
\label{table:configuration}
\end{table*}


\textbf{Hardware:} Our FedDIP framework and experiments are implemented and conducted on \textit{GeForce RTX 3090s} GPUs in the institution's HPC environment.

\subsection{Performance Under Extreme Sparsity}
To demonstrate the performance of FedDIP and other baseline methods under extreme sparsity, we set target $s_{p}=0.9$ for both \textit{Fashion-MNIST} and \textit{CIFAR10} tasks and $s_{p}=0.8$ for the \textit{CIFAR100} task. Notably, as $s_{p}=0.9$ causes divergence during the training of \textit{AlexNet} with SNIP, we 
adjust $s_{p}$ to $0.8$ for SNIP in this particular case.

\subsubsection{Accuracy} Figures \ref{fig:lenet_acc}, \ref{fig:alexnet_acc}, and \ref{fig:resnet18_acc} demonstrate that FedDIP surpasses other baselines in achieving the highest \textit{top-1} accuracy (ratio of the correctly classified images) while maintaining the same extreme target sparsity. 
As indicated in Table \ref{table:model_comparison}, to attain target sparsity of $s_{p}=0.9$ and $s_{p}=0.8$ respectively, FedDIP only compromises \textit{LeNet-5} and \textit{ResNet-18} model accuracy by $1.24\%$ and $1.25\%$, respectively. 
For \textit{AlexNet}, FedDIP can even improve model performance $0.7\%$, compared with FedAvg with $s_{p}=0.9$.

\subsubsection{Cumulative Communication \& Training Cost}
To make a fair comparison of cumulative communication cost during training (amount of information exchanged in MB) 
w.r.t. a fixed budget, we showcase the relationship between communication cost and accuracy. 
Figures \ref{fig:lenet_acc_cost}, \ref{fig:alexnet_acc_cost}, \ref{fig:resnet18_acc_cost}, and specifically 
Table \ref{table:communication_efficiency} present a comprehensive overview, emphasizing that FedDIP, when provided with adequate communication cost (budget),
effectively prunes the model across all experiments outperforming the other models. This indicates the trade-off between model performance and communication/training cost. 
FedDIP demonstrates comparable communication efficiency to other baselines, principally due to the minimal decrement in model performance. Through our experiments, it is evidenced that FedDIP achieves optimally pruned models under 
conditions of extreme sparsity, while incurring less or equivalent communication costs compared to FedAvg. 
Even in the early stages (i.e., in restricted budget cases), FedDIP manages to match the communication efficiency of other pruning methods in the \textit{CIFAR} experiments. This underscores the capacity of our approach to effectively balance model performance and communication expenditure. All in all, FedDIP introduces 
\textit{only} minor computational overhead due to the incremental regularization, while achieving high accuracy compared to baselines. This computational requirement is on par with that of PruneTrain, PruneFL, and SNIP, given the same sparsity at each epoch. A slight increase in computational cost can be justified by the improvements achieved in the final model performance considering extremely high sparsity. The size of the pruned CNN models (Table \ref{table:model_comparison}) has been significantly reduced ($\sim$ 1 order of magnitude) from the un-pruned models in FedAvg. 

\begin{figure}
     \centering
     \begin{subfigure}[b]{0.36\textwidth}
         \centering
         \includegraphics[width=\textwidth]{ 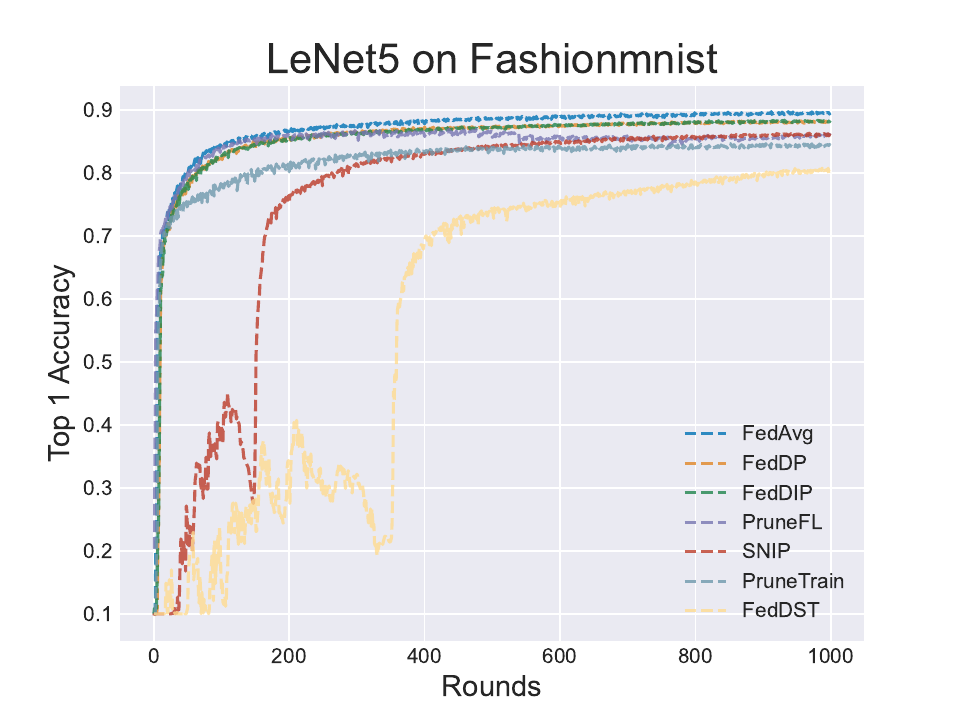}
         \caption{Test accuracy.}
         \label{fig:lenet_acc}
     \end{subfigure}
     \hfill
     \begin{subfigure}[b]{0.36\textwidth}
         \centering
         \includegraphics[width=\textwidth]{ 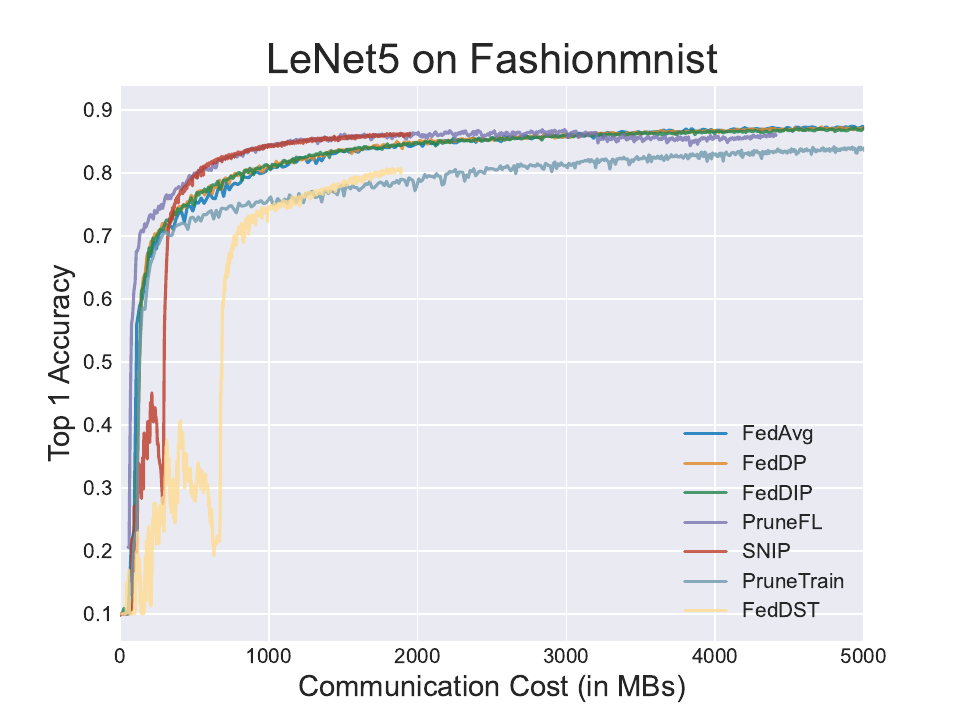}
         \caption{Test accuracy vs. communication budget.}
         \label{fig:lenet_acc_cost}
     \end{subfigure}
\caption{Fashion-MNIST experiment with LeNet-5.}
\label{fig:feddip_lenet5}
\end{figure}

\begin{figure}[h!t]
     \centering
     \begin{subfigure}[b]{0.36\textwidth}
         \centering
         \includegraphics[width=\textwidth]{ 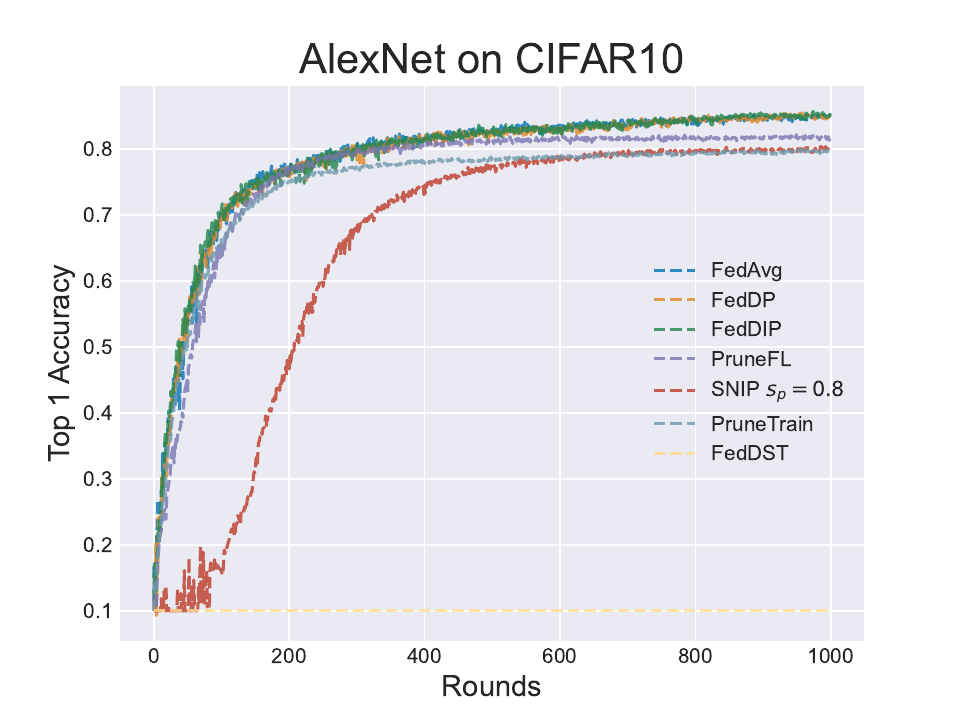}
         \caption{Test accuracy.}
         \label{fig:alexnet_acc}
     \end{subfigure}
     \hfill
     \begin{subfigure}[b]{0.36\textwidth}
         \centering
         \includegraphics[width=\textwidth]{ 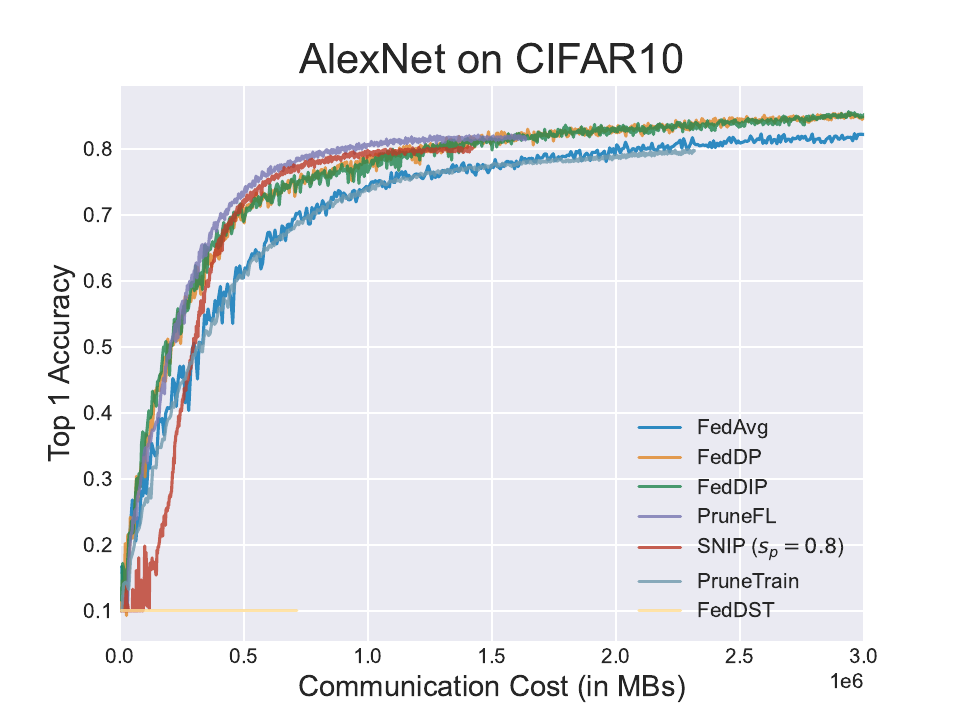}
         \caption{Test accuracy vs. communication budget.}
         \label{fig:alexnet_acc_cost}
     \end{subfigure}
\caption{CIFAR10 experiment with AlexNet.}
\label{fig:feddip_alexnet5}
\end{figure}

\begin{figure}
     \centering
     \begin{subfigure}[b]{0.36\textwidth}
         \centering
         \includegraphics[width=\textwidth]{ 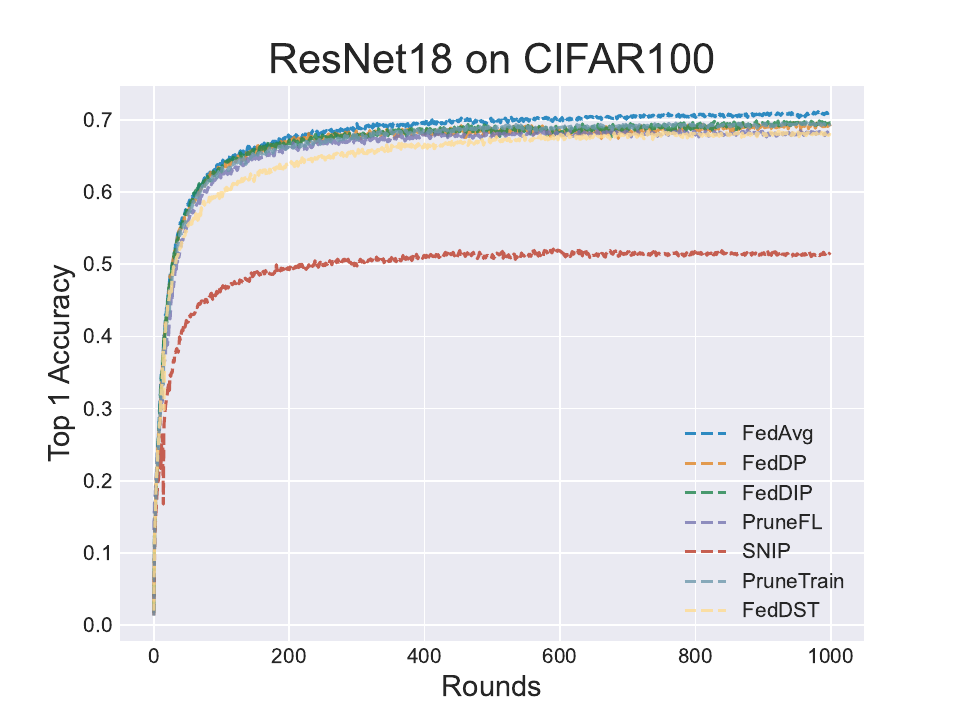}
         \caption{Test accuracy.}
         \label{fig:resnet18_acc}
     \end{subfigure}
     \hfill
     \begin{subfigure}[b]{0.36\textwidth}
         \centering
         \includegraphics[width=\textwidth]{ 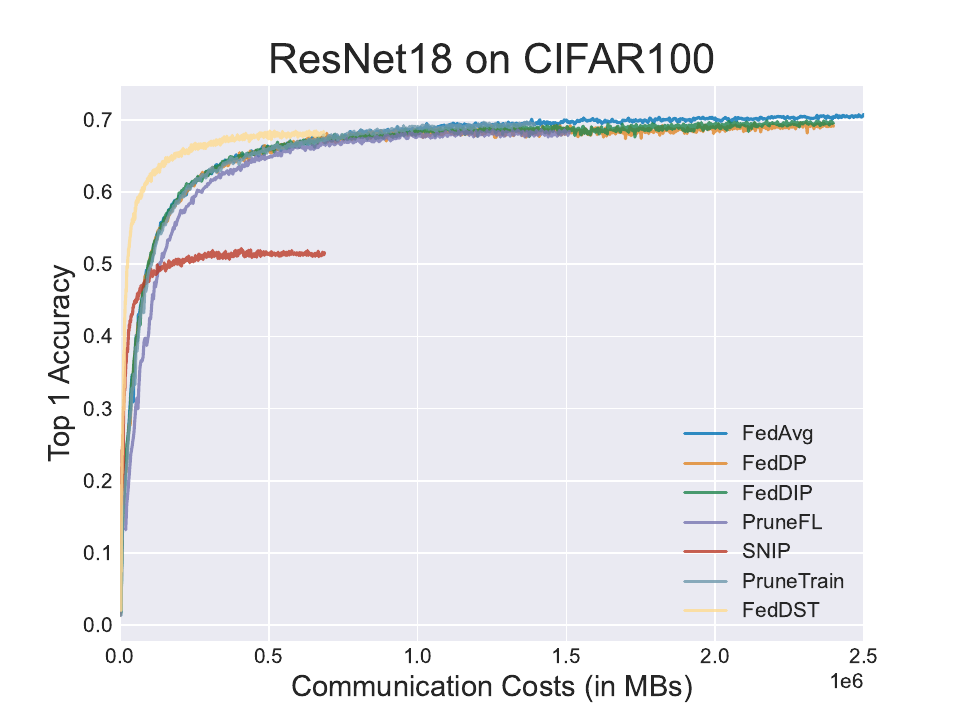}
         \caption{Test accuracy vs. communication budget.}
         \label{fig:resnet18_acc_cost}
     \end{subfigure}
\caption{CIFAR100 experiment with ResNet18.}
\label{fig:feddip_resnet18}
\end{figure}

\subsubsection{Experiments with non-i.i.d. data}
As shown in Table \ref{table:non_iid_results}, our methodology exhibits strong adaptability to FedProx (non-pruning), yielding commendable results on non-i.i.d. data. When juxtaposed with FedAvg, our approach manages to maintain comparable results even after pruning $90\%$ of model parameters, albeit at a slight trade-off of 1-2\% in model accuracy in the experiments with \textit{LeNet-5} and \textit{AlexNet}. Across a span of $T=1000$ rounds, FedDIP emerges as the superior performer in terms of \textit{top-1} accuracy, particularly at sparsity $s_{p}=0.8$ in \textit{ResNet-18}. This comprehensive suite of results underscores the adaptability of FedDIP in effectively managing non-i.i.d. cases, even in extreme sparsity.

\begin{table}[h]
\caption{Test Accuracy (\textit{top-1})}
\begin{center}
\begin{tabular}{|c|c|c|c|}
\hline
 &\multicolumn{3}{|c|}{\textbf{Model Performance ($\%$)$^{\mathrm{1}}$ with target sparsity $s_{p}$}}\\
\hline
\textbf{Model}& \textit{LeNet};$s_{p}=.9$ & \textit{AlexNet};$s_{p}=.9$& \textit{ResNet};$s_{p}=.8$ \\
\hline
\textit{FedAvg}& 89.50 (.09)  & 85.07 (.13)  & 70.92 (.10) \\
\textit{FedDP}    & 88.06 (.08)  & 84.81 (.18) & 69.23 (.14)   \\
\textit{FedDIP}  & \textbf{88.26} (0.09)  & \textbf{85.14} (.22) & \textbf{69.67} (.10)  \\
\textit{PruneFL}   & 86.00 (.10)  & 81.64 (.17) & 68.17 (.20)  \\
\textit{SNIP}   & 86.08 (.15)  & 80.10 (.15) & 51.46 (.11)  \\
\textit{PruneTrain}   & 84.36 (.10)  & 79.73 (.10)& 69.39 (.08)  \\
\textit{FedDST}  &    80.37 (.20)         &      --      & 68.06 (.20)\\
\hline
\# param.(FedAvg) & 62K  & 23.3M & 11.2M  \\
\# param.(pruned)  & \textbf{6.1K} & \textbf{2.3M} & \textbf{2.2M} \\
\hline
\multicolumn{4}{l}{$^{\mathrm{1}}$ Mean accuracy; standard deviation in `()'.}\\
\end{tabular}
\label{table:model_comparison}
\end{center}
\end{table}

\begin{table}[h]
\caption{Communication Efficiency}
\begin{center}
\begin{tabular}{|c|c|c|c|}
\hline
 &\multicolumn{3}{|c|}{\textbf{Model Performance ($\%$) with communication budget}}\\
\hline
\textbf{Case}& \textit{LeNet-5}$^{\mathrm{(1)}}$ & \textit{AlexNet}$^{\mathrm{(2)}}$ & \textit{ResNet-18}$^{\mathrm{(3)}}$  \\
\hline
\textit{FedAvg}& 86.76  & 78.98  & 70.06 \\
\textit{FedDP}    & 86.54  & 82.29 & 69.10    \\
\textit{FedDIP}  & \textbf{86.62}   & \textbf{82.58} & \textbf{69.57}  \\
\textit{PruneFL}   & 85.57  & 81.73  & 68.4  \\
\textit{SNIP}   & 86.32  & 80.11 & 51.63  \\
\textit{PruneTrain}   & 82.68  & 78.16 & 69.42  \\
\textit{FedDST}       &  80.37      &   --   &    68.06   \\        
\hline
\multicolumn{4}{l}{Communication budget $^{\mathrm{(1)}}$$4\cdot10^{3}$MB, $^{\mathrm{(2)}}$$1.8\cdot10^{6}$MB, $^{\mathrm{(3)}}$$2\cdot10^{6}$MB} \\
\end{tabular}
\label{table:communication_efficiency}
\end{center}
\end{table}
\begin{table}[h]
\caption{Extension to non-i.i.d. data}
\begin{center}
\begin{tabular}{|c|c|c|c|}
\hline
 &\multicolumn{3}{|c|}{\textbf{Model Performance$^{\mathrm{1}}$ ($\%$) (non-i.i.d. case)}}\\
\hline
\textbf{Case}& \textit{LeNet-5} & \textit{AlexNet} & \textit{ResNet-18}  \\
\hline
\textit{FedAvg}& 76.42(0.28)  & 61.59 (0.73) & 16.44 (0.49) \\
\textit{FedProx}    & \textbf{76.63}(0.34)  & \textbf{65.74} (0.26) & 18.48 (0.91)   \\
\textit{FedDIP+Prox}  &  74.49 (0.09)  & 60.47 (0.52) & \textbf{19.22} (0.8)  \\
\hline
\multicolumn{4}{l}{$^{\mathrm{1}}$ Mean of the highest five \textit{top-1} test accuracy during $T$ rounds.} \\
\end{tabular}
\label{table:non_iid_results}
\end{center}
\end{table}

\subsection{FedDIP Sparsity Analysis}

\subsubsection{Layerwise sparsity}
Figure \ref{fig:feddip_sparsity_vis} shows the sparsity \textit{per} layer of ResNet-18 ($s_{p} = 0.8$), LeNet-5 ($s_{p}=0.9$), and AlexNet ($s_{p}=0.9$). Notably, the first layers of all models are the least pruned ($0.3 \leq s \leq 0.4$), which is attributed to their significant role in general feature extraction. Furthermore, there is a correlation between the number of weights per layer and the 
corresponding sparsity level. This stems from the initial ERK distribution, which allocates a higher degree of sparsity to layers containing more weights, although we adopt global magnitude pruning in a later process. Such correlation is remarkable in both convolutional and fully-connected layers of the models. 
In convolutional layers, the correlations are found to be perfectly linear for \textit{LeNet-5} with a correlation coefficient $\varrho \simeq 1$, for \textit{AlexNet} we obtain $\varrho = 0.86$, while for \textit{ResNet-18} $\varrho = 0.8$. 
For fully-connected layers, since only one exists in \textit{ResNet-18}, we obtain $\varrho = (0.91, 0.82)$ for \textit{LeNet-5}, \textit{AlexNet}, respectively.  These findings highlight the dependency of layerwise sparsity and the number of weights per layer, reflecting 
the influence of the ERK distribution in FedDIP's initialization.

\begin{figure}[h!]
     \centering
     \begin{subfigure}[b]{0.35\textwidth}
         \centering
         \includegraphics[width=\textwidth]{ 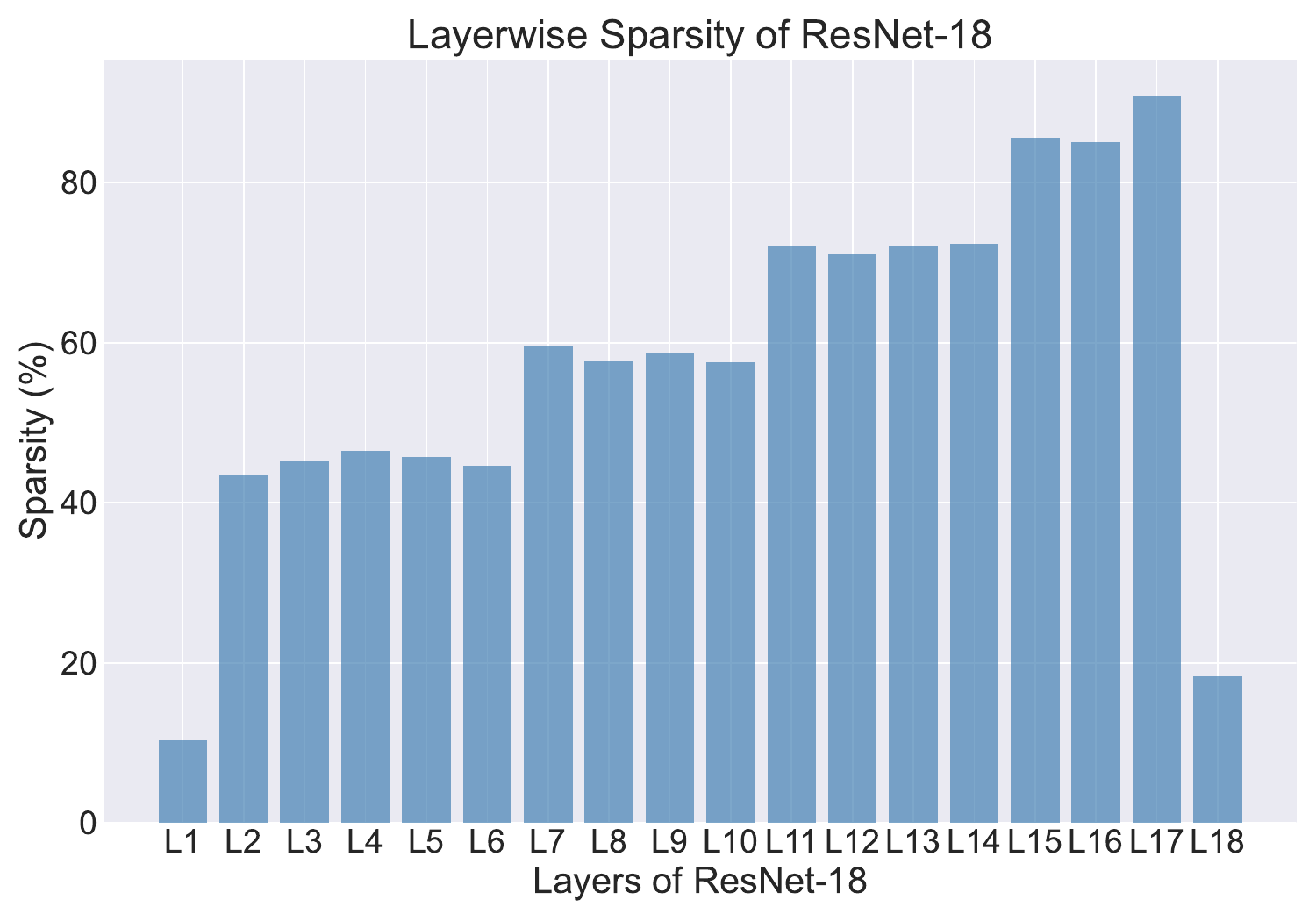}
         \caption{Distribution of layer sparsity; ResNet-18.}
         \label{fig:resnet_sp_visualization}
     \end{subfigure}
     \begin{subfigure}[b]{0.35\textwidth}
         \centering
         \includegraphics[width=\textwidth]{ 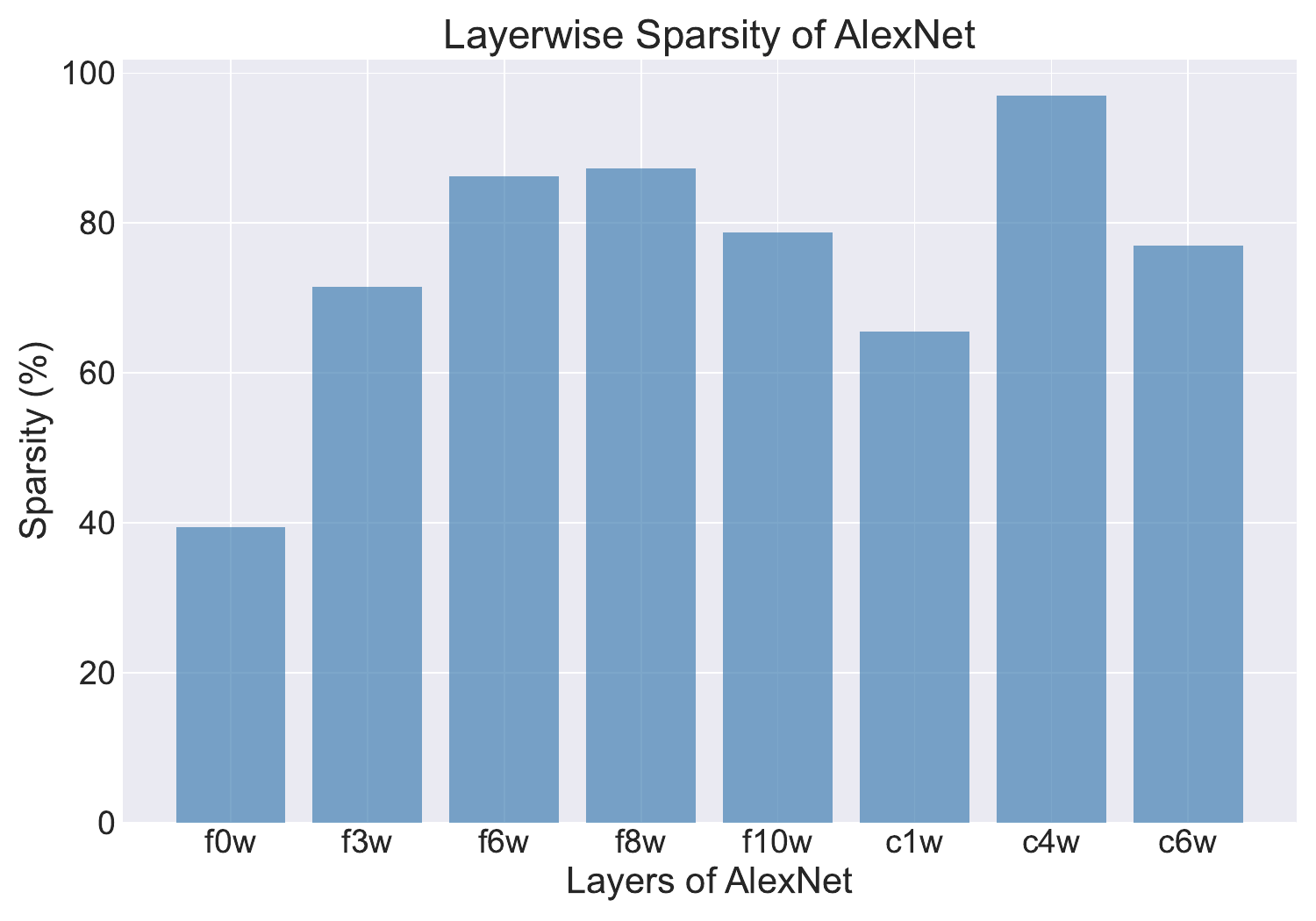}
         \caption{Distribution of layer sparsity; AlexNet.}
         \label{fig:alexnet_sp_visualization}
     \end{subfigure}
     \begin{subfigure}[b]{0.35\textwidth}
         \centering
         \includegraphics[width=\textwidth]{ 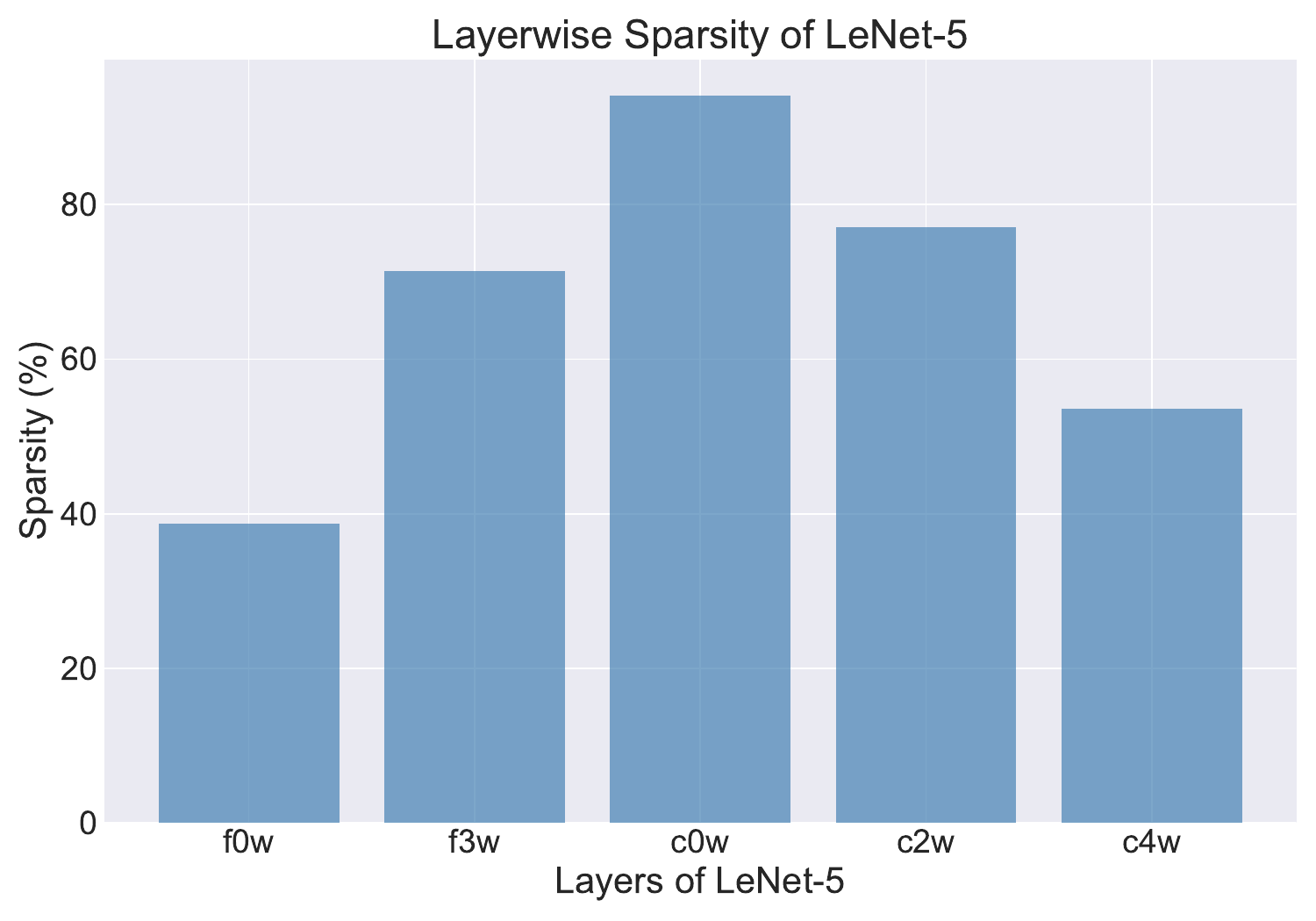}
         \caption{Distribution of layer sparsiy; LeNet-5.}
         \label{fig:lenet_sp_visualization}
     \end{subfigure}
\caption{Layerwise pruning sparsity; \textit{f0w} stands for (\textit{f})eatures layer, layer index (e.g, 0), and (\textit{w})eights, respectively. \textit{c} stands for the fully-connected classifier layer (the same notation is used for other layers). ResNet-18 consists of 18 pruning layers.}
\label{fig:feddip_sparsity_vis}
\end{figure}

\subsubsection{FedDIP in extreme sparsity}
We examine the efficiency of FedDIP under varying conditions of extreme sparsity. 
For \textit{Fashion-MNIST} and \textit{CIFAR10} experiments, we investigate two 
additional extreme sparsity levels $s_{p}=0.95$ and $s_{p}=0.99$, and for \textit{CIFAR100} experiments, we investigate $s_{p}=0.9$ and $s_{p}=0.95$. These conditions provide a robust assessment of FedDIP's performance across a range of extreme sparsity. As shown in Figure \ref{fig:extreme_sparsities}, under extreme sparsity like $0.95$ and $0.99$, the largest drops $\Delta$ in classification accuracy are only $\Delta = 6.97\%$, $\Delta = 5.03\%$, and $\Delta = 8.08\%$, respectively. This also comes with \textit{further} 90\%, 89\%, and 74\% reduction on LeNet-5, Alex-Net, and ResNet-18 model sizes, respectively. 
This indicates (i) FedDIP's efficiency in storing and managing trained and pruned models as well as (ii) efficiency in inference tasks (after training) due to relatively small models. All in all, the pruned DNN models' performance is relatively high with small accuracy drops 
and high model compression (92\%) across different tasks.
\begin{figure}[h!]
     \centering
     \begin{subfigure}[b]{0.57\textwidth}
         \centering
         \includegraphics[width=0.9\textwidth]
         { 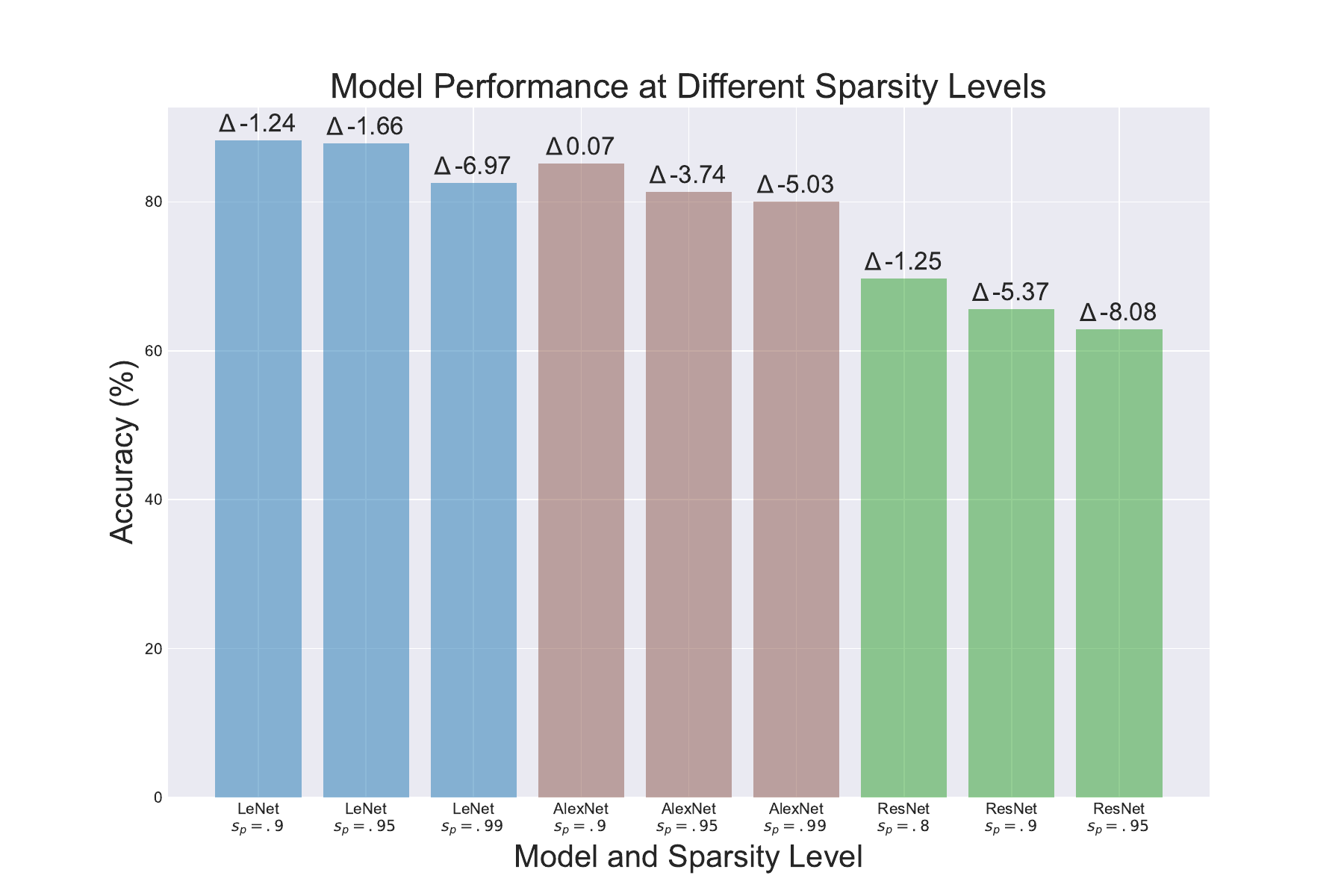}
         \label{extreme_sparsities}
     \end{subfigure}
\caption{FedDIP performance on extreme sparsity values.}
\label{fig:extreme_sparsities}
\end{figure}

\section{Conclusions}
\label{sec:conclusion}
We propose FedDIP, a novel FL framework with dynamic pruning and incremental regularization achieving highly accurate and extremely sparse DNN models. FedDIP gradually regularizes sparse DNN models obtaining extremely compressed models that maintain baseline accuracy and ensure controllable communication overhead. FedDIP is a data-free initialization method based on ERK distribution. We provide a theoretical convergence analysis of FedDIP and evaluate it across different DNN structures. FedDIP achieves comparable and higher accuracy against FL baselines and state-of-the-art FL-based model pruning approaches, respectively, over extreme sparsity using benchmark data sets (i.i.d. \& non-i.i.d. cases). Our agenda includes addressing heterogeneity in personalized FL environments.

\section*{Acknowledgement}
The authors would like to express their sincere gratitude to Dr. Fani Deligianni for her invaluable insights and discussions during peer communications.

This work is partially funded by the EU Horizon Grant `Integration and Harmonization of Logistics Operations' TRACE (\#101104278) and 'National Natural Science Foundation of China' (NSFC) under Grant \#72201093.
\bibliographystyle{IEEEtran}
\bibliography{reuse_in_fl}

\newpage
\appendix

\section{Proof for FedDIP}
\label{sec:appendix}
\newcommand\equalA{\mathrel{\overset{\makebox[0pt]{\mbox{\normalfont\tiny\sffamily {\small \textcircled{\small a}}}}}{=}}}
\newcommand\leqA{\mathrel{\overset{\makebox[0pt]{\mbox{\normalfont\tiny\sffamily {\small \textcircled{\small a}}}}}{\leq}}}
\newcommand\equalB{\mathrel{\overset{\makebox[0pt]{\mbox{\normalfont\tiny\sffamily {\small \textcircled{\small b}}}}}{=}}}
\newcommand\leqB{\mathrel{\overset{\makebox[0pt]{\mbox{\normalfont\tiny\sffamily {\small \textcircled{\small b}}}}}{\leq}}}
\newcommand\equalC{\mathrel{\overset{\makebox[0pt]{\mbox{\normalfont\tiny\sffamily {\small \textcircled{\small c}}}}}{=}}}
\newcommand\leqC{\mathrel{\overset{\makebox[0pt]{\mbox{\normalfont\tiny\sffamily {\small \textcircled{\small c}}}}}{\leq}}}
\newcommand\equalseven{\mathrel{\overset{\makebox[0pt]{\mbox{\normalfont\tiny\sffamily {\small \textcircled{\small 7}}}}}{=}}}
\newcommand\leqeight{\mathrel{\overset{\makebox[0pt]{\mbox{\normalfont\tiny\sffamily {\small \textcircled{\small 8}}}}}{\leq}}}
Before the proof of convergence Theorem~\ref{theorem:1}, we stress that: 
\begin{multline}\label{eq:prune+_origin}
    \mathbb{E}[f(\boldsymbol{\bar{\omega}}^{\prime(t+1)})-f(\boldsymbol{\bar{\omega}}^{\prime(t)})] = \mathbb{E}[f(\boldsymbol{\bar{\omega}}^{\prime(t+1)})]-\mathbb{E}[f(\boldsymbol{\bar{\omega}}^{(t+1)})]\\
    +\mathbb{E}[f(\boldsymbol{\bar{\omega}}^{(t+1)})]-\mathbb{E}[f(\boldsymbol{\bar{\omega}}^{\prime(t)})],
\end{multline}
where mask update only happens at server and $\boldsymbol{\bar{\omega}}^{\prime(t)}$ is the global model received by nodes at time $t$, as the start of the local model training phase.
\begin{lemma}
\label{lemma1}
Given any mask function $\mathbf{m}:=\{0,1\}^{n\times p}$ for pruning, the Frobenius norm of model weight/gradients matrix $\boldsymbol{\omega}$ is greater than or equal to the pruned one $\mathbf{m}\odot \boldsymbol{\omega}$, i.e., 
\begin{equation}\label{lemma1_1}
    \|\boldsymbol{\omega}\|\geq \|\mathbf{m}\odot \boldsymbol{\omega}\|.
\end{equation}
\end{lemma}

\begin{proof}
According to the definition of Frobenius norm, we have
\begin{multline*}
     \|\mathbf{m}\odot \boldsymbol{\omega}\|^{2} = Tr([\mathbf{m}\odot \boldsymbol{\omega}]^{T}\cdot[\mathbf{m}\odot \boldsymbol{\omega}]) \\
     = \sum_{i=1}^{n}\sum_{j=1}^{p}\lvert m_{ij}\boldsymbol{\omega}_{ij}\rvert^{2}\leq \sum_{i=1}^{n}\sum_{j=1}^{p}\lvert \boldsymbol{\omega}_{ij}\rvert^{2}=\|\boldsymbol{\omega}\|^{2}.
\end{multline*}
\end{proof}

\textit{Lemma} \ref{lemma1} provides the foundation that the quality of pruning $\delta_{t}$, respectively, are in the range $[0,1]$.
\begin{lemma}
\label{lemma2}
Given the Definition \ref{def:quality_prune} and Assumption \ref{assumption:mu_Lipschitz}, the effect of pruning on pruned model weights at server ($\delta_{t+1}$) is bounded as:
\begin{align}
            \mathbb{E}[f(\boldsymbol{\bar{\omega}}^{\prime(t+1)})]-\mathbb{E}[f(\boldsymbol{\bar{\omega}}^{t+1})]\leq \mu\mathbb{E}[\sqrt{\delta_{t+1}}\|\boldsymbol{\bar{\omega}}^{t+1}\|]
\end{align}
\end{lemma}
\begin{proof}
\begin{multline}\label{eq:lemma2_1}
        \mathbb{E}[f(\boldsymbol{\bar{\omega}}^{\prime(t+1)})]-\mathbb{E}[f(\boldsymbol{\bar{\omega}}^{t+1})]\leq \mu\mathbb{E}[\|\boldsymbol{\bar{\omega}}^{\prime(t+1)}-\boldsymbol{\bar{\omega}}^{t+1}\|]\\
        = \mu\mathbb{E}[\sqrt{\delta_{t+1}}\|\boldsymbol{\bar{\omega}}^{t+1}\|].
\end{multline}
\end{proof}

\begin{lemma}\label{lemma3}
Under the definitions provided in Section \ref{sec:CA} and Assumptions $1$ \& $4$, $\mathbb{E}[f(\boldsymbol{\bar{\omega}}^{t+1})]-\mathbb{E}[f(\boldsymbol{\bar{\omega}}^{\prime(t)})]$ is bounded by:
\begin{multline}\label{eq:lemma3_0}
   \mathbb{E}[f(\boldsymbol{\bar{\omega}}^{t+1})]-\mathbb{E}[f(\boldsymbol{\bar{\omega}}^{\prime(t)})]
   \leq \frac{(\gamma-1)L^{2}\eta_{t}^{2}+\eta_{t}^{2}L}{2K}\sum_{n=1}^{N}\rho_{n}\sigma^{2}_{n}+\\
\frac{(\gamma-1)\gamma E_{l}\eta_{t}^{2}L^{2}}{2}\sum_{k=t_{c}+1}^{t_{c}+E_{l}}\|\sum_{n=1}^{N}\rho_{n}\mathbf{v}_{n}^{\prime(k)}\|^{2}
-\frac{\eta_{t}}{2}\|\nabla f(\boldsymbol{\bar{\omega}}^{\prime(t)})\|^{2}+\\
\frac{\gamma\eta_{t}^{2}L-\eta_{t}}{2}\|\sum_{n=1}^{N}\rho_{n}\Tilde{\mathbf{v}}_{n}^{\prime(t)}\|^{2}.
\end{multline}
\end{lemma}

\begin{proof}
Firstly, according to the Assumption \ref{assumption:alpha_convex}, we have
\begin{multline}\label{eq:lemma3_1}
    \mathbb{E}[f(\boldsymbol{\bar{\omega}}^{t+1})]-\mathbb{E}[f(\boldsymbol{\bar{\omega}}^{\prime(t)})]\\
    \leq \mathbb{E}[<\boldsymbol{\bar{\omega}}^{t+1}-\boldsymbol{\bar{\omega}}^{\prime(t)},\nabla f(\boldsymbol{\bar{\omega}}^{\prime(t)})>]+\frac{L}{2}\mathbb{E}\|\boldsymbol{\bar{\omega}}^{t+1}-\boldsymbol{\bar{\omega}}^{\prime(t)}\|^{2}\\
    \equalA{\frac{\eta_{t}^{2}L}{2}\mathbb{E}\|\Tilde{\mathbf{v}}^{\prime(t)}\|^{2}}- \mathbb{E}[<\eta_{t}\Tilde{\mathbf{v}}^{\prime(t)},\nabla f(\boldsymbol{\bar{\omega}}^{\prime(t)})>],
\end{multline}
where $\equalA{}$ holds because of (\ref{eq:feddip_glob}). Following the proof structure in \cite{ca2021}, we give the boundary for $- \mathbb{E}[<\eta_{t}\Tilde{\mathbf{v}}^{\prime(t)},\nabla f(\boldsymbol{\bar{\omega}}^{\prime(t)})>]$ and $\mathbb{E}\|\Tilde{\mathbf{v}}^{\prime(t)}\|^{2}$, respectively. 

According to the variance formula and the definition of $\Bar{\mathbf{v}}^{\prime(t)}$, we can expand $\mathbb{E}\|\Tilde{\mathbf{v}}^{\prime(t)}\|^{2}$ as follows:
\begin{multline}\label{eq:lemma3_3}
    \mathbb{E}\|\Tilde{\mathbf{v}}^{\prime(t)}\|^{2} = \mathbb{E}\|\Tilde{\mathbf{v}}^{\prime(t)}-\mathbb{E}(\Tilde{\mathbf{v}}^{\prime(t)})\|^{2}+[\mathbb{E}(\Tilde{\mathbf{v}}^{\prime(t)})]^{2}\\
    =\mathbb{E}\|\Tilde{\mathbf{v}}^{\prime(t)}-\Bar{\mathbf{v}}^{\prime(t)}\|^{2} +\|\Bar{\mathbf{v}}^{\prime(t)}\|^{2}.
\end{multline}
Recall that $P_{t}$ is the selection probability of clients, thus, one can obtain:
\begin{multline}\label{eq:lemma3_4}
    \|\Bar{\mathbf{v}}^{\prime(t)}\|^{2} = \left\|\frac{1}{K}\mathbb{E}_{P_{t}}[\sum_{n\in P_{t}}\mathbf{v}_{n}^{\prime(t)}]\right\|^{2}
    \leq \frac{1}{K}\mathbb{E}_{P_{t}}\left[\sum_{n\in P_{t}}\|\mathbf{v}_{n}^{\prime(t)}\|^{2}\right]
\end{multline}
The last inequality holds due to \textit{Jenson's Inequality}.
Similarly, based on the definitions in (\ref{eq:pruned_gradients}), we have:
\begin{multline}\label{eq:lemma3_5}
    \mathbb{E}[\|\Tilde{\mathbf{v}}^{\prime(t)}-\Bar{\mathbf{v}}^{\prime(t)}\|^{2}] = \mathbb{E}[\|\mathbb{E}_{P_{t}}[\frac{1}{K}\sum_{n \in P_{t}}\Tilde{\mathbf{v}}_{n}^{\prime(t)}-\frac{1}{K}\sum_{n \in P_{t}}\mathbf{v}_{n}^{\prime(t)}]\|^{2}]\\
    = \frac{1}{K^{2}}\{\mathbb{E}[\mathbb{E}_{P_{t}}[\sum_{n\in P_{t}}\|\Tilde{\mathbf{v}}_{n}^{\prime(t)}-\mathbf{v}_{n}^{\prime(t)}\|^{2}]]\\
    + \sum_{i\not =j}<\Tilde{\mathbf{v}}_{i}^{\prime(t)}-\mathbf{v}_{i}^{\prime(t)},\Tilde{\mathbf{v}}_{j}^{\prime(t)}-\mathbf{v}_{j}^{\prime(t)}>\}\\
    = \frac{1}{K}\mathbb{E}[\sum_{n=1 }^{N}\rho_{n}\|\Tilde{\mathbf{v}}_{n}^{\prime(t)}-\mathbf{v}_{n}^{\prime(t)}\|^{2}].
\end{multline}
Substituting \eqref{eq:lemma3_4} and \eqref{eq:lemma3_5} into \eqref{eq:lemma3_3}, and according to the \textit{Lemma} $1$, Definition \ref{def:non_iid} and Assumption $4$, we obtain that: 
\begin{multline}
\label{eq:lemma3_6}
\mathbb{E}\|\Tilde{\mathbf{v}}^{\prime(t)}\|^{2} \leq  \frac{1}{K}\mathbb{E}[\sum_{n\in P_{t}}\|\mathbf{v}_{n}^{\prime(t)}\|^{2}] + \frac{1}{K^{2}}\mathbb{E}[\sum_{n\in P_{t} }\|\Tilde{\mathbf{v}}_{n}^{\prime(t)}-\mathbf{v}_{n}^{\prime(t)}\|^{2}]\\
\leq \sum_{n=1}^{N}\rho_{n}\|\mathbf{v}_{n}^{\prime(t)}\|^{2}+\frac{1}{K}\mathbb{E}[\sum_{n=1 }^{N}\|\Tilde{\mathbf{v}}_{n}^{\prime(t)}-\mathbf{v}_{n}^{\prime(t)}\|^{2}]\\
\leq \sum_{n=1}^{N}\rho_{n}\|\mathbf{v}_{n}^{\prime(t)}\|^{2} + \frac{1}{K}\sum_{n=1}^{N}\rho_{n}\sigma_{n}^{2}\\
\leq \gamma \sum_{n=1}^{N}\|\rho_{n}\mathbf{v}_{n}^{\prime(t)}\|^{2} + \frac{1}{K}\sum_{n=1}^{N}\rho_{n}\sigma_{n}^{2}.
\end{multline}
Next, we provide the boundary for $- \mathbb{E}[<\eta_{t}\Tilde{\mathbf{v}}^{\prime(t)},\nabla f(\boldsymbol{\bar{\omega}}^{\prime(t)})>]$ . 
\begin{multline}
\label{eq:lemma3_7}
- \mathbb{E}[\langle\Tilde{\mathbf{v}}^{\prime(t)},\nabla f(\boldsymbol{\bar{\omega}}^{\prime(t)})\rangle] = - \langle\mathbb{E}[\frac{1}{K}\sum_{n\in P_{t}}\Tilde{\mathbf{v}}_{n}^{\prime(t)}],\nabla f(\boldsymbol{\bar{\omega}}^{\prime(t)})\rangle \\
= - \langle\mathbb{E}[\sum_{n=1}^{N}\rho_{n}\Tilde{\mathbf{v}}_{n}^{\prime(t)}],\nabla f(\boldsymbol{\bar{\omega}}^{\prime(t)})\rangle\\
\equalA -\frac{1}{2}\|\nabla f(\boldsymbol{\bar{\omega}}^{\prime(t)})\|^{2}-\frac{1}{2}\|\sum_{n=1}^{N}\rho_{n}\Tilde{\mathbf{v}}_{n}^{\prime(t)}\|^{2}+\frac{1}{2}\|\nabla f(\boldsymbol{\bar{\omega}}^{\prime(t)})-\sum_{n=1}^{N}\rho_{n}\Tilde{\mathbf{v}}_{n}^{\prime(t)}\|^{2}\\
= -\frac{1}{2}\|\nabla f(\boldsymbol{\bar{\omega}}^{\prime(t)})\|^{2}-\frac{1}{2}\|\sum_{n=1}^{N}\rho_{n}\Tilde{\mathbf{v}}_{n}^{\prime(t)}\|^{2}+\\
\frac{1}{2}\|\sum_{n=1}^{N}\rho_{n}(\nabla f_{n}(\boldsymbol{\bar{\omega}}^{\prime(t)})-\nabla f_{n}(\boldsymbol{\omega}_{n}^{\prime(t)}))\|^{2}\\
\leqB -\frac{1}{2}\|\nabla f(\boldsymbol{\bar{\omega}}^{\prime(t)})\|^{2}-\frac{1}{2}\|\sum_{n=1}^{N}\rho_{n}\Tilde{\mathbf{v}}_{n}^{\prime(t)}\|^{2}+\\
\frac{L^{2}}{2}\sum_{n=1}^{N}\rho_{n}\|(\boldsymbol{\bar{\omega}}^{\prime(t)}-\boldsymbol{\omega}_{n}^{\prime(t)})\|^{2},
\end{multline}

where the condition \textcircled{a} is satisfied due to the relation $ab = \frac{a^{2}+b^{2}-(a-b)^{2}}{2}$, while condition \textcircled{b} is satisfied by virtue of Assumption \ref{assumption:alpha_convex}.
Our approach to the proof utilizes a similar structure to the one found in the reference \cite{ca2021}, owing to the fact that the global mask $\mbox{m}_{t}$ remains consistent throughout the local training process. The specifics of this proof methodology are explicated as follows. Let $t_{c}=[\frac{t}{E_{l}}]E_{l}$ be the start time of local training. Then $\boldsymbol{\bar{\omega}}^{\prime(t)}$ and $\boldsymbol{\omega}_{n}^{\prime(t)}$ can be written as
\begin{multline}
\label{eq:lemma3_8}
\boldsymbol{\bar{\omega}}^{\prime(t)} = (\boldsymbol{\bar{\omega}}^{\prime(t_{c})}-\frac{1}{K}\sum_{n\in P_{t}}\sum_{k=t_{c}+1}^{t-1}\eta_{k}\Tilde{\mathbf{v}}_{n}^{\prime(k)})\odot \mbox{m}_{t},
\end{multline}
\begin{multline}
\label{eq:lemma3_9}
\boldsymbol{\omega}^{\prime(t)}_{n} = (\boldsymbol{\bar{\omega}}^{\prime(t_{c})}-\sum_{k=t_{c}+1}^{t-1}\eta_{k}\Tilde{\mathbf{v}}_{n}^{\prime(k)})\odot \mbox{m}_{t}.
\end{multline}
According to Lemma \ref{lemma1}, Assumption \ref{assumption:bounded_gradient}, Definition \ref{def:non_iid} and Eq. $(50)-(54)$ in \cite{ca2021}, while taking expectation of $P_{t}$, we can obtain
\begin{multline}
\label{eq:lemma3_10}
\sum_{n=1}^{N}\rho_{n}\|(\boldsymbol{\bar{\omega}}^{\prime(t)}-\boldsymbol{\omega}_{n}^{\prime(t)})\|^{2}\\
\leq (\gamma-1)[\frac{1}{K}\sum_{n=1}^{N}\sum_{k=t_{c}+1}^{t-1}\eta_{k}^{2}\rho_{n}\sigma^{2}_{n}+\gamma E_{l}\sum_{k=t_{c}+1}^{t-1}\eta_{k}^{2}\|\sum_{n=1}^{N}\rho_{n}\mathbf{v}_{n}^{\prime(k)}\|^{2}].
\end{multline}
Consider decreasing learning rate $\eta_{k}^{2}\leq\eta_{t}$ with $t_{c}+1<k<t$ and substitute Eq. (\ref{eq:lemma3_10}) into (\ref{eq:lemma3_7}), we have 
\begin{multline}
\label{eq:lemma3_11}
- \mathbb{E}[\langle\Tilde{\mathbf{v}}^{\prime(t)},\nabla f(\boldsymbol{\bar{\omega}}^{\prime(t)})\rangle] \leq \frac{(\gamma-1)L^{2}\eta_{t}}{2K}\sum_{n=1}^{N}\rho_{n}\sigma^{2}_{n}+\\
\frac{(\gamma-1)\gamma E_{l}\eta_{t}L^{2}}{2}\sum_{k=t_{c}+1}^{t_{c}+E_{l}}\|\sum_{n=1}^{N}\rho_{n}\mathbf{v}_{n}^{\prime(k)}\|^{2}
-\frac{1}{2}\|\nabla f(\boldsymbol{\bar{\omega}}^{\prime(t)})\|^{2}\\
-\frac{1}{2}\|\sum_{n=1}^{N}\rho_{n}\Tilde{\mathbf{v}}_{n}^{\prime(t)}\|^{2}.
\end{multline}
Substituting \eqref{eq:lemma3_6} and \eqref{eq:lemma3_11} into \eqref{eq:lemma3_1} finishes the proof.
\end{proof}

\textbf{Proof of Theorem $1$:}
\begin{proof}
Recall \eqref{eq:prune+_origin}, we sum up the theoretical results of \textit{Lemma} \ref{lemma2} and \ref{lemma3}. 
\begin{multline}\label{eq:thrm1_0}
   \mathbb{E}[f(\boldsymbol{\bar{\omega}}^{\prime(t+1)})]-\mathbb{E}[f(\boldsymbol{\bar{\omega}}^{\prime(t)})]
   \leq \frac{(\gamma-1)L^{2}\eta_{t}^{2}+\eta_{t}^{2}L}{2K}\sum_{n=1}^{N}\rho_{n}\sigma^{2}_{n}+\\
\frac{(\gamma-1)\gamma E_{l}\eta_{t}^{2}L^{2}}{2}\sum_{k=t_{c}+1}^{t_{c}+E_{l}}\|\sum_{n=1}^{N}\rho_{n}\mathbf{v}_{n}^{\prime(k)}\|^{2}
-\frac{\eta_{t}}{2}\|\nabla f(\boldsymbol{\bar{\omega}}^{\prime(t)})\|^{2}+\\
\frac{\gamma\eta_{t}^{2}L-\eta_{t}}{2}\|\sum_{n=1}^{N}\rho_{n}\Tilde{\mathbf{v}}_{n}^{\prime(t)}\|^{2}+\mu\mathbb{E}[\sqrt{\delta_{t+1}}\|\boldsymbol{\bar{\omega}}^{t+1}\|].
\end{multline}
While $\eta_{t}\leq\frac{1}{tL}$ and Assumption \ref{assumption_bounded_weighted_gradient}, Equation \eqref{eq:thrm1_0} can be expressed as:
\begin{multline}\label{eq:thrm1_1}
    \mathbb{E}[f(\boldsymbol{\bar{\omega}}^{\prime(t+1)})]-\mathbb{E}[f(\boldsymbol{\bar{\omega}}^{\prime(t)})]
   \leq \frac{(\gamma-1)L^{2}\eta_{t}^{2}+\eta_{t}^{2}L}{2K}\sum_{n=1}^{N}\rho_{n}\sigma^{2}_{n}+\\
\frac{(\gamma-1)\gamma E_{l}^{2}\eta_{t}^{2}L^{2}G^{2}}{2}+\mu\mathbb{E}[\sqrt{\delta_{t+1}}\|\boldsymbol{\bar{\omega}}^{t+1}\|.
\end{multline}
Denote $\chi=\frac{(\gamma-1)L^{2}+L}{2K}\sum_{n=1}^{N}\rho_{n}\sigma^{2}_{n}+
\frac{(\gamma-1)\gamma E_{l}^{2}L^{2}G^{2}}{2}$, \eqref{eq:thrm1_1} is written as
\begin{multline}\label{eq:thrm1_2}
    \frac{\eta_{t}}{2}\|\nabla f(\boldsymbol{\bar{\omega}}^{\prime(t)})\|^{2}
   \leq \mathbb{E}[f(\boldsymbol{\bar{\omega}}^{\prime(t)})]-\mathbb{E}[f(\boldsymbol{\bar{\omega}}^{\prime(t+1)})]\\+\chi\eta_{t}^{2}
+\mu\mathbb{E}[\sqrt{\delta_{t+1}}\|\boldsymbol{\bar{\omega}}^{t+1}\|.
\end{multline}
 After taking the average of (\ref{eq:thrm1_2}) over time $T$, we have the results as follows after rearranging. Note that we assume that the model will converge to a stable point regarded as the optimum $f^{*}$.   
\begin{multline}\label{eq:thrm1_4}
    \frac{1}{T}\sum_{t=1}^{T}\frac{1}{tL}\|\nabla f(\boldsymbol{\bar{\omega}}^{\prime(t)})\|^{2}\leq \frac{1}{T}\mathbb{E}(f(\boldsymbol{\omega}_{1})-f^{*}) +\\
    \frac{2}{T}\sum_{t=1}^{T}[\mu\mathbb{E}[\sqrt{\delta_{t+1}}\|\boldsymbol{\bar{\omega}}^{t+1}\|
    + \chi\eta_{t}^{2}].
\end{multline}
Since $\frac{1}{T}\sum_{t=1}^{T}\frac{1}{T}\|\nabla f(\boldsymbol{\bar{\omega}}^{\prime(t)})\|^{2}\leq \frac{1}{T}\sum_{t=1}^{T}\frac{1}{t}\|\nabla f(\boldsymbol{\bar{\omega}}^{\prime(t)})\|^{2}$, \eqref{eq:thrm1_4} can be expressed as
\begin{multline}
    \frac{1}{T}\sum_{t=1}^{T}\|\nabla f(\boldsymbol{\bar{\omega}}^{\prime(t)})\|^{2}\leq 2L\mathbb{E}(f(\boldsymbol{\omega}_{1})-f^{*}) +\\2L\sum_{t=1}^{T}[\mu\mathbb{E}[\sqrt{\delta_{t+1}}\|\boldsymbol{\bar{\omega}}^{t+1}\|]+\frac{\pi^{2}}{3L^{2}}\chi,
\end{multline}
where it is known that $\sum_{t=1}^{T}\frac{1}{t^{2}}=\frac{\pi^{2}}{6}$.
\end{proof}

\vspace{12pt}
\end{document}